\documentclass[]{article}
\usepackage{proceed2e}

\usepackage{times}
\usepackage[T1]{fontenc}    % use 8-bit T1 fonts
\usepackage{hyperref}       % hyperlinks
\usepackage{url}            % simple URL typesetting
\usepackage{booktabs}       % professional-quality tables
\usepackage{amsfonts}       % blackboard math symbols
\usepackage{nicefrac}       % compact symbols for 1/2, etc.
\usepackage{microtype}      % microtypography
\usepackage{epstopdf}
\usepackage{graphicx}
\usepackage{amsmath}
\usepackage{amssymb}
\usepackage{mathptmx}
\usepackage{graphicx}
\usepackage{subfigure}
\usepackage[textwidth=2cm,colorinlistoftodos]{todonotes}
\usepackage{xspace}
\usepackage{latexsym}
\usepackage{amsmath,amssymb,graphicx,xspace}
\usepackage{times}   % assumes new font selection scheme installed
\usepackage{multirow}
\usepackage{algorithm}
\usepackage{amsthm}
\usepackage{amsmath}
\usepackage[algo2e]{algorithm2e}
\usepackage[font=small,skip=5pt]{caption}
\usepackage[noend]{algpseudocode}
\usepackage{booktabs}

\usepackage{lipsum}

\newcommand\blfootnote[1]{%
  \begingroup
  \renewcommand\thefootnote{}\footnote{#1}%
  \addtocounter{footnote}{-1}%
  \endgroup
}

\let\OLDthebibliography\thebibliography
\renewcommand\thebibliography[1]{
  \OLDthebibliography{#1}
  \setlength{\parskip}{0pt}
  \setlength{\itemsep}{0pt plus 0.3ex}
}

\expandafter\def\expandafter\normalsize\expandafter{%
    \normalsize
    \setlength\abovedisplayskip{5pt}
    \setlength\belowdisplayskip{5pt}
    \setlength\abovedisplayshortskip{5pt}
    \setlength\belowdisplayshortskip{5pt}
}
\usepackage{color,comment,rotating,subfigure,url,xspace} % author packages
\usepackage{tikz}
\usetikzlibrary{arrows,shadows,shapes,backgrounds,decorations,snakes,fit}
\usepackage{todonotes}
\usepackage{enumerate}
\usepackage{lipsum}
\usepackage{titlesec}

\titlespacing\section{0pt}{2pt plus 2pt minus 2pt}{2pt plus 2pt minus 2pt}
\titlespacing\subsection{0pt}{2pt plus 2pt minus 2pt}{2pt plus 2pt minus 2pt}
\titlespacing\subsubsection{0pt}{0pt plus2pt minus 2pt}{0pt plus 2pt minus 2pt}
\titlespacing\paragraph{0pt}{0pt plus 2pt minus 2pt}{0pt plus 2pt minus 2pt}

\theoremstyle{definition}

\newtheorem{theorem}{Theorem}
\newtheorem{proposition}{Proposition}

\DeclareMathOperator*{\E}{\mathbb{E}}
\title{Determinantal Point Processes for Mini-Batch Diversification}

\author{
Cheng Zhang\\
Disney Research\\
Pittsburgh, PA, USA\\
\tt\small{cheng.zhang@disneyresearch.com}
\And Hedvig Kjellstr\"{o}m\\
KTH Royal Institute of Technology\\
Stockholm, Sweden\\
\tt\small{hedvig@kth.se}\\
\And Stephan Mandt \\
Disney Research\\
Pittsburgh, PA, USA\\
\tt\small{stephan.mandt@disneyresearch.com}
}
\begin{document}
\maketitle
\begin{abstract}
We study a mini-batch diversification scheme for stochastic gradient descent (SGD). While classical SGD relies on uniformly sampling data points to form a mini-batch, we propose a non-uniform sampling scheme based on the Determinantal Point Process (DPP). The DPP relies on a similarity measure between data points and gives low probabilities to mini-batches which contain redundant data, and higher probabilities to mini-batches with more diverse data.
This simultaneously balances the data and leads to stochastic gradients with lower variance. We term this approach Diversified Mini-Batch SGD (DM-SGD).  We show that regular SGD and a biased version of stratified sampling emerge as special cases. Furthermore, DM-SGD generalizes stratified sampling to cases where no discrete features exist to bin the data into groups. We show experimentally that our method results more interpretable and diverse features in unsupervised setups, and in better classification accuracies in supervised setups.  
\end{abstract}
\blfootnote{In the proceedings of Uncertainty in Artificial Intelligence (UAI 2017).}
\vspace{-8pt}
\section{INTRODUCTION}

Stochastic gradient descent (SGD) is one of the most important algorithms for scalable machine learning \cite{bottou2010large,robbins1985convergence,mandt2017stochastic}. SGD optimizes an objective function by successively following noisy estimates of its gradient based on mini-batches from a large underlying dataset. We usually assure that this gradient is unbiased, meaning that the expected stochastic gradient equals the true gradient. When combined with a suitably decreasing learning rate schedule, the algorithm converges to a local optimum of the objective~\cite{bottou2010large}.

Often we are not interested in learning an unbiased estimator of the gradient, but are rather willing to introduce some bias. There are many reasons for why this might be the case. First, biased SGD schemes such as momentum \cite{polyak1964some}, iterate averaging \cite{schmidt2013minimizing}, or preconditioning \cite{duchi2011adaptive, kingma2014adam, tieleman2012lecture, zeiler2012adadelta} may reduce the stochastic gradient noise or ease the optimization problem, and therefore often lead to faster convergence. Another reason is that we may decide to actively select samples based on their relevance or difficulty levels such as boosting \cite{freund1995desicion}, or because we believe that our dataset is in some respect imbalanced \cite{he2009learning}. In this paper, we propose and investigate a biased mini-batch subsampling scheme for imbalanced data.

Real-world data sets are naturally imbalanced. 
%%SM: I don't understand this argument.---Cheng: it is the argument from diversified priors, which means the true population is imbalanced. But to learn a useful represetation. We still want to rebalance the data.
%, which is caused by two factors. First, the frequency of true information is long tailed (imbalanced) \cite{kwok2012priors,xie2015diversifying}, for example, 
For instance, the sports topic appears more often in the news than biology; the internet contains more images of young people than of senior people, and Youtube has more videos of cats  than of bees or ants.
Aiming to maximize the probability of generating such training data, machine learning models will refine the dominant information with redundancy but ignore the important but scarce data. For example, a model trained on Youtube data might be very sensitive to different cats but unable to recognize ants. We may therefore decide to try to learn on a more balanced data set by actively selecting diversified mini-batches.

The currently most common tool for mini-batch diversification is stratified sampling \cite{neyman1934two,zhao2014accelerating}. In this approach, one groups the data into a finite set of \emph{strata} based on discrete or continuous features such as a label or cluster assignment. To re-balance the data set, the data can then be subsampled such that each stratum occurs with equal probability in the mini-batch (in the following, we refer to this method as \emph{biased stratified sampling}). Unfortunately, the data are not always amenable to biased stratified sampling because discrete features may not exist, or the data may not be unambiguously clustered. Instead of subsampling based on discrete strata, it would be desirable to diversify the mini-batch based on a soft similarity measure between data points. As we show in this paper, this can be achieved using Determinantal Point Processes (DPPs) \cite{kulesza2012determinantal}.

The DPP is a point process which mimics repulsive interactions between samples. Being based on a similarity matrix between the data points, a draw from a DPP yields diversified subsets of the data. The main contribution of this paper is using this mechanism to diversify the mini-batches in stochastic gradient-based learning and analyzing this setup theoretically. In more detail, our main achievements are:
%thereby as follows:
 \begin{itemize}
     \vspace{-10pt}
     \item We present a  mini-batch diversification scheme based on DPPs for stochastic gradient algorithms.  
     This approach requires a similarity measure among data points, which can be constructed using low-level features of the data. Since the sampling strategy is independent of the learning objective, diversified mini-batches can be precomputed in parallel and reused for different learning tasks. Our approach applies to both supervised and unsupervised models.
     \vspace{-5pt}
     \item  We prove that our method is a  generalization of stratified sampling and i.i.d.~mini-batch sampling. Both cases emerge for specific similarity kernels of the data.  
     \vspace{-7pt}
     \item  We theoretically analyze the conditions under which the variance of the DM-SGD gradient gets reduced. We also give an unbiased version of DM-SGD which optimizes the original objective without re-balancing the data. 
     \vspace{-5pt}
     \item  
     %SM: I don't think that any of this is important to mention since these are just terms that we introduce.
     %We introduce the concept of diversified risk in comparison with the empirical risk. For the special case of stochastic variational inference \cite{hoffman13stochastic}, we optimize the balanced-population posterior. 
     We carry out extensive experiments on several models and datasets. Our approach leads to faster learning and higher classification accuracies in deep supervised learning. For topic models we find that that the resulting document features are more interpretable and are better suited for subsequent supervised learning tasks. 
     \vspace{-10pt}
 \end{itemize}

Our paper is structured as follows. In Section \ref{sec:related} we list related work. Section \ref{sec:method} discusses our main concepts of a diversifed risk, and discuses the DM-SGD method. Section \ref{sec:theory} discusses theoretical properties of our approach such as variance reduction. Finally, in Section \ref{sec:exp}, we give empirical evidence that our approach leads to higher classification accuracy and better feature extractions than i.i.d. sampling.

\section{RELATED WORK}
\label{sec:related}
%There is a large volume of related work of stochastic optimization, imbalanced learning and data subsampling. 
We revisit the most relevant prior work based on the following aspects. \emph{Diversification and Stratification} comprises methods which aim at re-balancing the empirical distribution of the data. \emph{Variance reduction} summarizes stochastic gradient methods that aim at faster convergence by reducing the stochastic gradient noise. Finally, we list related applications and extensions of \emph{determinantal point processes}.

%Firstly, we review several population learning works  which associated with the population balance aspect of our method. Then, we focus on works that relates to stochastic optimization with mini-batches. This group of work mainly share the common advancement of fast convergence with our proposed method. In the end, related work on DPP is summarized shortly since it is an important component of our work. 

%\paragraph{Population learning}
\paragraph{Diversification and stratification.}~
Since our method suggests to diversify the mini-batches by of non-uniform subsampling from the data, it relates to stratification methods.

Stratification~\cite{neyman1934two,mckay1979comparison} assumes that the data decomposes into disjoint sub-datasets, called strata. These are formed based on certain criteria such as a class-label. Instead of uniformly sampling from the whole dataset, each stratum is sub-sampled independently, which reduces the variance of the estimator of interest.

Stratified sampling has been suggested as a variance reduction method for stochastic gradient algorithms \cite{fu2017CPSGMCMC,zhao2014accelerating}. 
%This subsampling strategy can be biased or unbiased, meaning that samples from thin classes (classes with few data points) may be effectively associated with a higher weight. 
If one subsamples the same number of data points from every stratum to form a mini-batch as in \cite{zhao2014accelerating}, one naturally balances the training procedure. This approach was also used in \cite{facebookcode}. Our work relates closely to this type of biased stratified sampling. It is different in that it does not rely on discrete strata, but only requires a measure a measure of similarity between data points to achieve a similar effect. This applies more broadly.

%\paragraph{Fast convergence of stochastic methods}
\paragraph{Variance reduction.}~Besides re-balancing the dataset, our approach also reduces the variance of the stochastic gradients. Several ways of variance reduction of stochastic gradient algorithms have been proposed, an important class relying on control variates~\cite{mandt2014smoothed, paisley2012variational, ranganath2014black, wang2013variance,   salimans2014using}.
%Several prior arts have provided theoretical profs of the coupling relationship between mini-batch selection and learning rates \cite{balles2016coupling,friedlander2012hybrid}. In general, lower variance of the gradient estimates from mini-batch allows larger learning rate. Commonly, 
%Firstly, increasing the mini-batch size corespondents to decreasing of variance of the gradients \cite{friedlander2012hybrid}. Hence, to accelerate the learning procedure, prior studies either focus on adapting the learning rate based on general statistics of the mini-batches, such as  \cite{kingma2014adam, meng2015objective,ranganath2013adaptive,tieleman2012lecture, zeiler2012adadelta}, or focus on the mini-batch selection. 
%For mini-batch selection, we can group the related work by whether the sampling is random (naive sampling). With naive sampling, prior work either focus on optimizing the learning gain through optimizing the mini-batch size \cite{balles2016coupling,byrd2012sample,de2017big} or reduce the variances with correction terms \cite{wang2013variance}. 
A second class of methods relies on non-uniform sampling of mini-batches \cite{csiba2016importance,fu2017CPSGMCMC, perekrestenko2017faster, schmidt2015non,zhao2014accelerating,zhao2015stochastic}. None of these methods rely on similarity measures between data points. 
%Our method is a non-uniform sampling method which balances the population and accelerates the convergence at the same time. Our proposed methods thus relates closely with this group of work.  Among all non-uniform sampling methods, importance sampling based methods as been applied mainly on traditional optimization settings \cite{csiba2016importance, perekrestenko2017faster, zhao2015stochastic}. Although effective, it is very hard to apply to general machine learning models, because the computation of the sampling probability is highly coupled with high-dimensional model parameters \cite{fu2017CPSGMCMC} which make it non practical. 

Our approach is most closely related to clustering-based sampling (CBS) \cite{fu2017CPSGMCMC} and stratified sampling (StS) \cite{zhao2014accelerating}. 
%SM: is the following statement true? This sounds like a boild claim. I simply added a sentence to the last paragraph to have the desired effect.
%(Compared to importance sampling \cite{csiba2016importance, perekrestenko2017faster, zhao2015stochastic}, CBS and StS are more general and computationally efficient \cite{fu2017CPSGMCMC}.) 
StS applies stratified sampling to SGD and builds on pre-specified strata. For every stratum, the same number of data points are uniformly selected, and then re-weighted according to the size of the stratum to make the sampling scheme un-biased. CBS uses a similar strategy, but does not require a pre-speficied set of strata.   
Instead, the strata are formed by pre-clustering the raw data with k-means. (Thus, if the data are clustered based on a class label, CBS is identical to StS.) The problem is that the data are not always amenable to clustering. Second, both StS and CBS ignore the within-cluster variations between data points. In contrast, our approach relies on a continuous measure of similarity between samples. We furthermore show that it is a strict generalization of both setups for particular choices of similarity kernels.

%Clustering is an on-going research topic by itself and k-means clustering is not able to provide satisfactory results especially when the data are imbalanced. Our method tackles this problem and can be viewed as generalization of StS and CBS on a sample level rather than on a strata level. 

\paragraph{Determinantal point processes.}~
The DPP \cite{kulesza2012determinantal,macchi1975coincidence} has been proposed \cite{kwok2012priors,lee2016individualness,xie2015diversifying} and  advanced \cite{affandi2013nystrom,li2016fast,li2015efficient} in the machine learning community in the recent years. 
%It is a probabilistic distribution of subsets from a fixed ground set. Briefly, the probability of a subset is positive proportional to the determinant of its corresponding similarity kernel matrix. In this way, DPP encourages diverse subsets, which makes DPP useful for, e.g., subset sampling
It has been applied in subset sampling \cite{kulesza2011k,li2015efficient} and results filtering \cite{lee2016individualness}.
%As we discussed in the previous section, imbalanced datasets commonly lead to under-represented models.

The DPP has also been used as a diversity-enhancing prior in Bayesian models \cite{kwok2012priors,xie2015diversifying}. 
In big data setups, the data may overwhelm the prior such that the strength of the prior has to scale with the number of data points; introducing a bias. The approach is furthermore constrained to hierarchical Bayesian models, while our approach applies to all empirical risk minimization problems.

Recently, efficient algorithms have been proposed to make sampling using the DPP more scalable. 
In the traditional formulation, mini-batch sampling  costs $\mathcal{O}(Nk^3)$, with an initial fixed cost of diagonalizing the similarity matrix \cite{kulesza2012determinantal}, where $N$ is the size of the data and $k$ is the size of the mini-batch. 
Recent scalable versions of the DPP rely on core-sets and low-rank approximations and scale more favorably \cite{affandi2013nystrom,li2015efficient}. These versions were used in our large-scale experiments. 

%SM: I don't think we need a summary for teh related work section.
%In summary, data diversification and convergence acceleration are two important research directions that our work closely connects to. Our proposed method achieves both goals without adding any running time cost but archives the goal of variance reduction for  balanced dataset.

%\paragraph{Other} 

%\begin{itemize}
%\item Boosting
%\item bootstrapping 
%\end{itemize}

\section{METHOD}
\label{sec:method}
Our method, DM-SGD, uses a version of the DPP for mini-batch sampling in stochastic gradient descent. We show that this balances the underlying data distribution and simultaneously accelerates the convergence due to variance reduction. We briefly revisit DPP first, and then introduce our mini-batch diversification method. Theoretical aspects are then discussed in Section \ref{sec:theory}.

%%%%%%%%%%%%%%%%%%%%%%%%%%%%%%%%%%%%%%%%%%%
%\subsection{Determinantal point processes revisited}
\subsection{DETERMINANTAL POINT PROCESSES}
%introduction for DPP
A point process is a collection of points randomly located in some mathematical space. The most prominent example is the Poisson process on the real line \cite{kingman1993poisson}, which models independently occurring events. In contrast, the DPP \cite{kulesza2012determinantal,macchi1975coincidence} models repulsive correlations between these points. 

In this paper, we restrict ourselves to a finite set of $N$ points. Denote by  $L\in {\mathbb R}^{N\times N}$ a similarity kernel matrix between these points, e.g.~based on spatial distances or some other criterion. $L$ is real, symmetric and positive definite, and its elements $L_{ij}$ are some appropriately defined measure of similarity between the $i_\text{th}$ and $j_\text{th}$ data. The DPP assigns a probability to subsampling any subset $Y$ of $\{1,\dots,N\}$, which is proportional to the determinant of the sub-matrix $L_Y$ of $L$ which indexes the subset,
\begin{equation}
\mathcal{P}(Y) = \frac{det(L_Y)}{det(L + I)} \propto det(L_Y).
\label{eq:DPP}
\end{equation}
%$L$ is commonly constructed with different choices of kernels on data feature.  In the example of text document, we can simply use a linear kernel $L(x_i,y_j) = x_i^{\rho}x_j'^{\rho}$ with TF-IDF feature vectors ($x$) of the documents. 
%DPP naturally diversifying the samples based on point-wise similarity measures since the data are more diversified with bigger determinant. 
For instance, if $Y=\{i,j\}$ consists of only two elements, then $\mathcal{P}(Y) \propto L_{ii} L_{jj} - L_{ij} L_{ji}$.
Because $L_{ij}$ and $L_{ji}$ measure the similarity between elements $i$ and $j$, being more similar lowers
the probability of co-occurrence. On the other hand, when the subset is very diverse, the determinant is bigger and correspondingly its co-occurrence is more likely. The DPP thus naturally diversifies the selection of subsets.

In this paper, we propose to use the DPP to diversify mini-batches. 
%In case of mini-batch sampling, a range of mini-batch sizes is preferred. 
In practice, the mini-batch size is usually constrained by empirical bounds or hardware restrictions. 
In this case, we want to use DPP conditioned on a given size $k$.  Therefore, a slightly modified version of the DPP is needed, which is called $k$-DPP~\cite{kulesza2011k}. It assigns probabilities to subsets of size $k$,
\begin{equation}
\mathcal{P}_L^k(Y) = \frac{det(L_Y)}{\sum_{|Y'|=k} det(L_{Y'})} .
\label{eq:kDPP}
\end{equation}
Apart from conditioning on the size of the subset of points, the $k$-DPP has the same diversification effect as the DPP~\cite{kulesza2011k}.
In order to have a fixed mini-batch size we use the $k$-DPP in this work. 

%%%%%%%%%%%%%%%%%%%%%%%%%%%%%%%%%%%%%%%%%%%%
%\subsection{Mini-batch diversification using the $k$-DPP}
\subsection{MINI-BATCH DIVERSIFICATION}
The diversifying property of the $k$-DPP makes it well-suited to diversify mini-batches. We first discuss our learning objective---the diversified risk. We then introduce our algorithm and qualitatively discuss its properties.

%%%%%%%%%%%%%%%%%%%%%%%%%%%%%%%%%%%
%\begin{figure}[t]
%\centering
%\scalebox{0.8}{\input{PopulationFi%g.tex}}
%\caption{Visualization of our working hypothesis. Our data is as a sample from an underlying population $D$, but we only observe an imbalanced realization $\hat{D}$. To better approximate $D$, we average our loss function over diversified subsets of $\hat{D}$, which results in the balanced population $D^*$, giving rise to the diversified risk $J^*$.}
%\label{fig:DataRelationship}
%\end{figure}

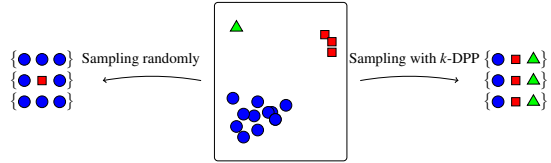
\begin{figure}[t]
\centering
\scalebox{0.47}{\pgfdeclarelayer{background}
\pgfdeclarelayer{foreground}
\pgfsetlayers{background,main,foreground}

\begin{tikzpicture}

\tikzstyle{Vnode}=[circle,draw=black, radius=3pt];
\tikzstyle{Snode}=[rectangle,draw=black, minimum size=3pt];
\tikzstyle{Tnode}=[regular polygon,draw=black, regular polygon sides=3, minimum size=1pt,inner sep=2pt];
\tikzstyle{surround} = [thick,draw=black,rounded corners=1mm];

%% data
\node [Vnode, fill=blue] at ( 0, 0) (){};
\node [Vnode, fill=blue] at ( 0.5, 0.1) (){};
\node [Vnode, fill=blue] at ( -0.5, -0.3) (){};
\node [Vnode, fill=blue] at ( 0.4, 0.1) (){};
\node [Vnode, fill=blue] at ( 0.9, 0.3) (){};
\node [Vnode, fill=blue] at ( 0.6, -0.1) (){};
\node [Vnode, fill=blue] at ( -0.3, 0.05) (){};
\node [Vnode, fill=blue] at ( 0.1, 0.4) (){};
\node [Vnode, fill=blue] at ( 0.1, -0.4) (){};
\node [Vnode, fill=blue] at ( 0.6, -0.1) (){};
\node [Vnode, fill=blue] at ( -0.3, -0.6) (b){};
\node [Vnode, fill=blue] at ( -0.6, 0.5) (){};

\node [Snode, fill=red] at ( 2.2, 1.8) (){};
\node [Snode, fill=red] at ( 2.0, 2.3) (r){};
\node [Snode, fill=red] at ( 2.2, 2.1) (){};

\node [Tnode, fill=green] at ( -0.5, 2.5) (g){};

\draw [->, thick] (3,1) arc (100:80:8) ;
\node [] at ( 4.53 , 1.6) (){\large Sampling with $k$-DPP};

\node [] at ( 4.2+2.4, 1.6) (){ \big \{ };
\node [] at ( 4.2+2.4, 1) (){\big \{};
\node [] at ( 4.2+2.4, 0.4) (){ \big \{};
\node [Vnode, fill=blue] at ( 4.5+2.4, 1.6) (){};
\node [Vnode, fill=blue] at ( 4.5+2.4, 1) (){};
\node [Vnode, fill=blue] at ( 4.5+2.4, 0.4) (){};
\node [Snode, fill=red] at ( 5+2.4, 1.6) (){};
\node [Snode, fill=red] at (5+2.4, 1) (){};
\node [Snode, fill=red] at ( 5+2.4, 0.4) (){};
\node [Tnode, fill=green] at ( 5.5+2.4, 1.6) (){};
\node [Tnode, fill=green] at (5.5+2.4, 1) (){};
\node [Tnode, fill=green] at ( 5.5+2.4, 0.4) (){};
\node [] at ( 6.8+1.4, 1.6) (){\big \}};
\node [] at ( 6.8+1.4, 1) (){\big \}};
\node [] at ( 6.8+1.4, 0.4) (){\big \}};

\draw [->, thick] (-1.5,1) arc (80:100:8) ;
\node [ ] at ( -3.2 , 1.6) (){\large Sampling randomly};

\node [] at ( 4.2-11, 1.6) (){ \big \{ };
\node [] at ( 4.2-11, 1) (){\big \{};
\node [] at ( 4.2-11, 0.4) (){ \big \{};
\node [Vnode, fill=blue] at ( 4.5-11, 1.6) (){};
\node [Vnode, fill=blue] at ( 4.5-11, 1) (){};
\node [Vnode, fill=blue] at ( 4.5-11, 0.4) (){};
\node [Vnode, fill=blue] at ( 5-11, 1.6) (){};
\node [Snode, fill=red] at ( 5-11, 1) (){};
\node [Vnode, fill=blue] at ( 5-11, 0.4) (){};
\node [Vnode, fill=blue] at ( 5.5-11, 1.6) (){};
\node [Vnode, fill=blue] at ( 5.5-11, 1) (){};
\node [Vnode, fill=blue] at ( 5.5-11, 0.4) (){};
\node [] at ( 5.8-11, 1.6) (){\big \}};
\node [] at ( 5.8-11, 1) (){\big \}};
\node [] at ( 5.8-11, 0.4) (){\big \}};

%%plate
\node[surround, inner sep = .5cm] (f_N) [fit = (g)(r)(b) ] {};

\end{tikzpicture}}
\caption{Sampling mini-batches using the $k$-DPP. For an imbalanced dataset, our method results in diversified mini-batches.}
\vspace{-13pt}
\label{fig:DPP}
\end{figure}

\begin{figure*}[t]
    \centering
    \subfigure[True population]{
    \includegraphics[width=3.7cm, height=3cm]{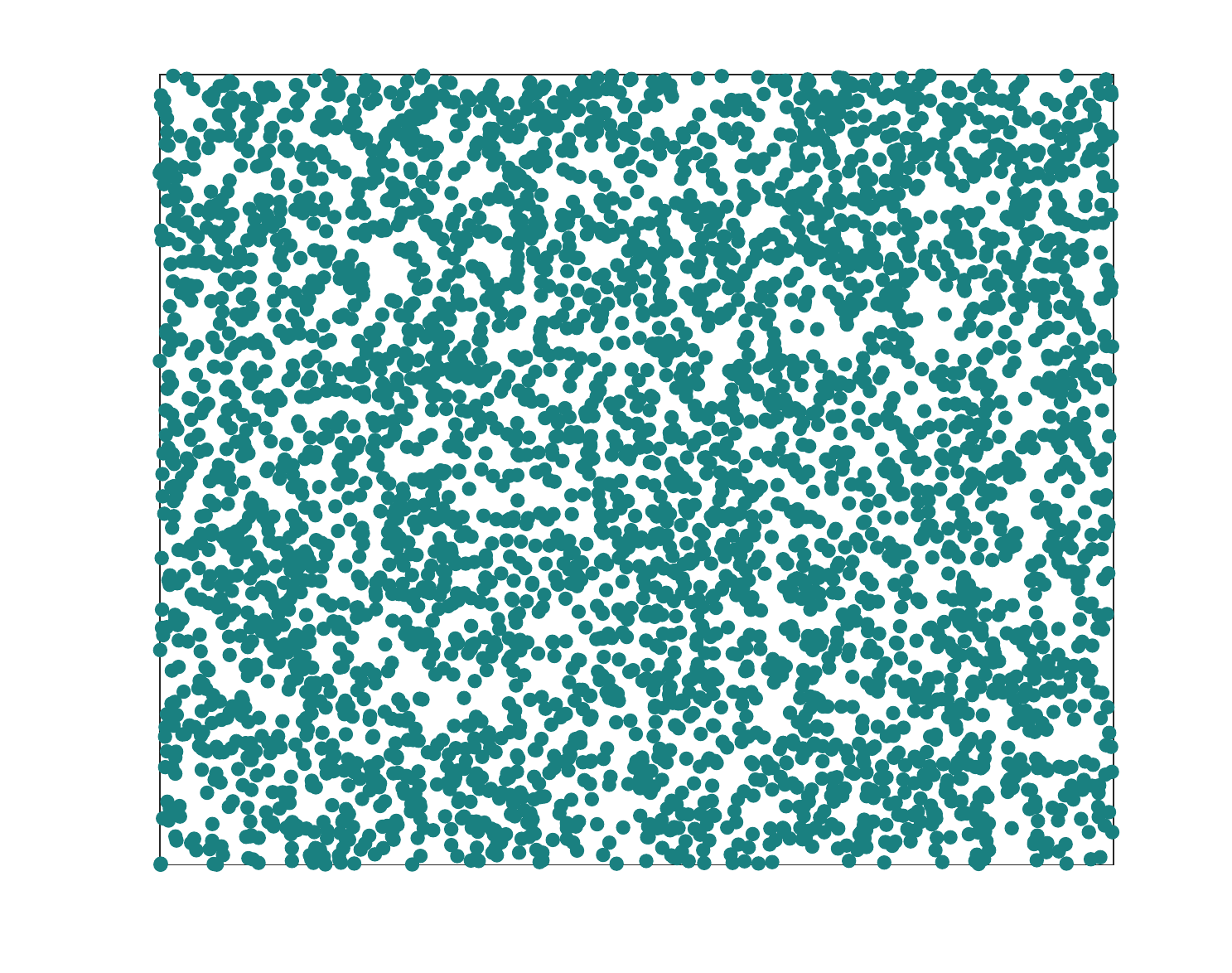}
    }
    \hspace{-20pt}
    \subfigure[Dataset]{
    \includegraphics[width=3.7cm, height=3cm]{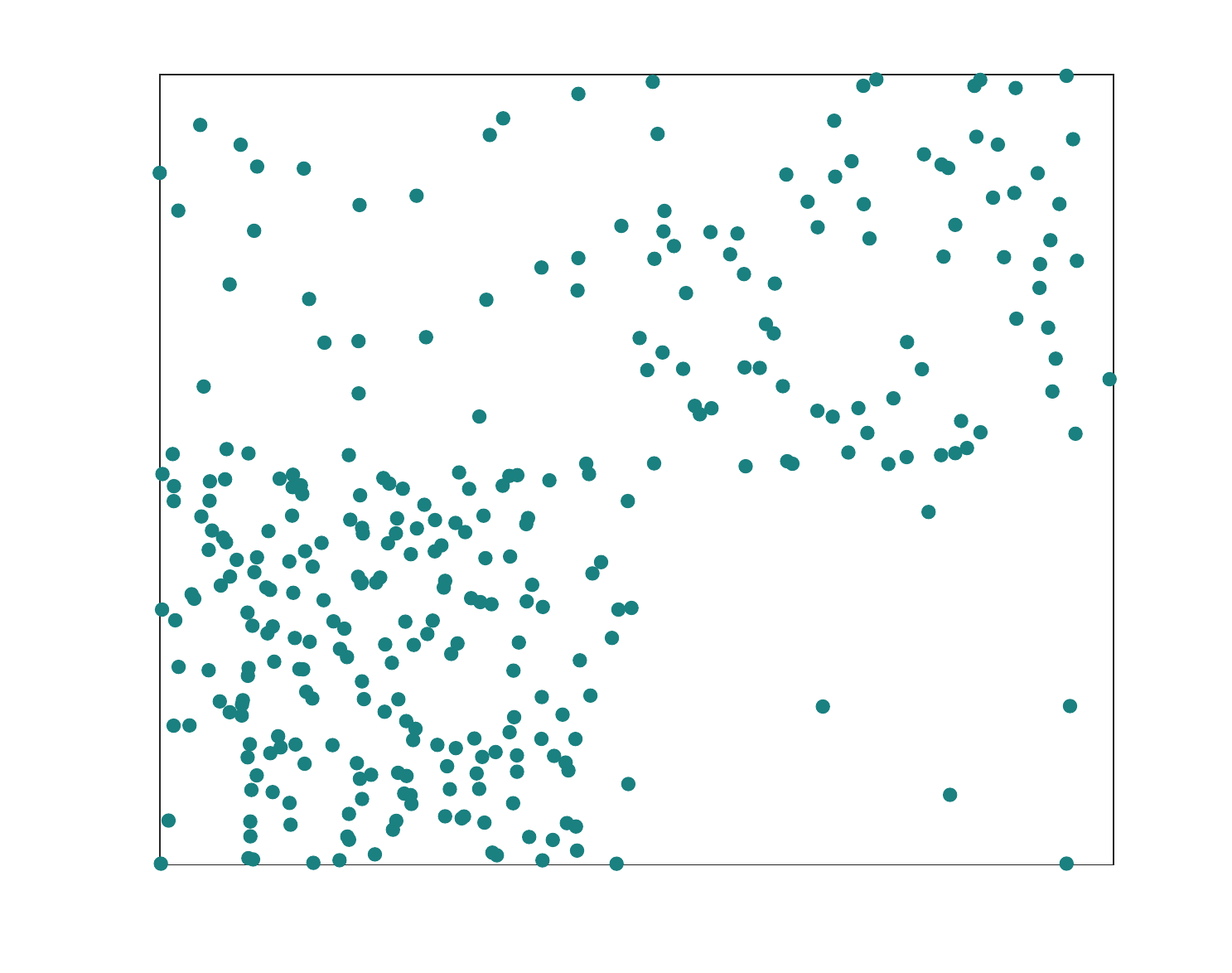}
    }
    \hspace{-20pt}
    \subfigure[Stratified Sampling]{
    \includegraphics[width=3.7cm, height=3cm]{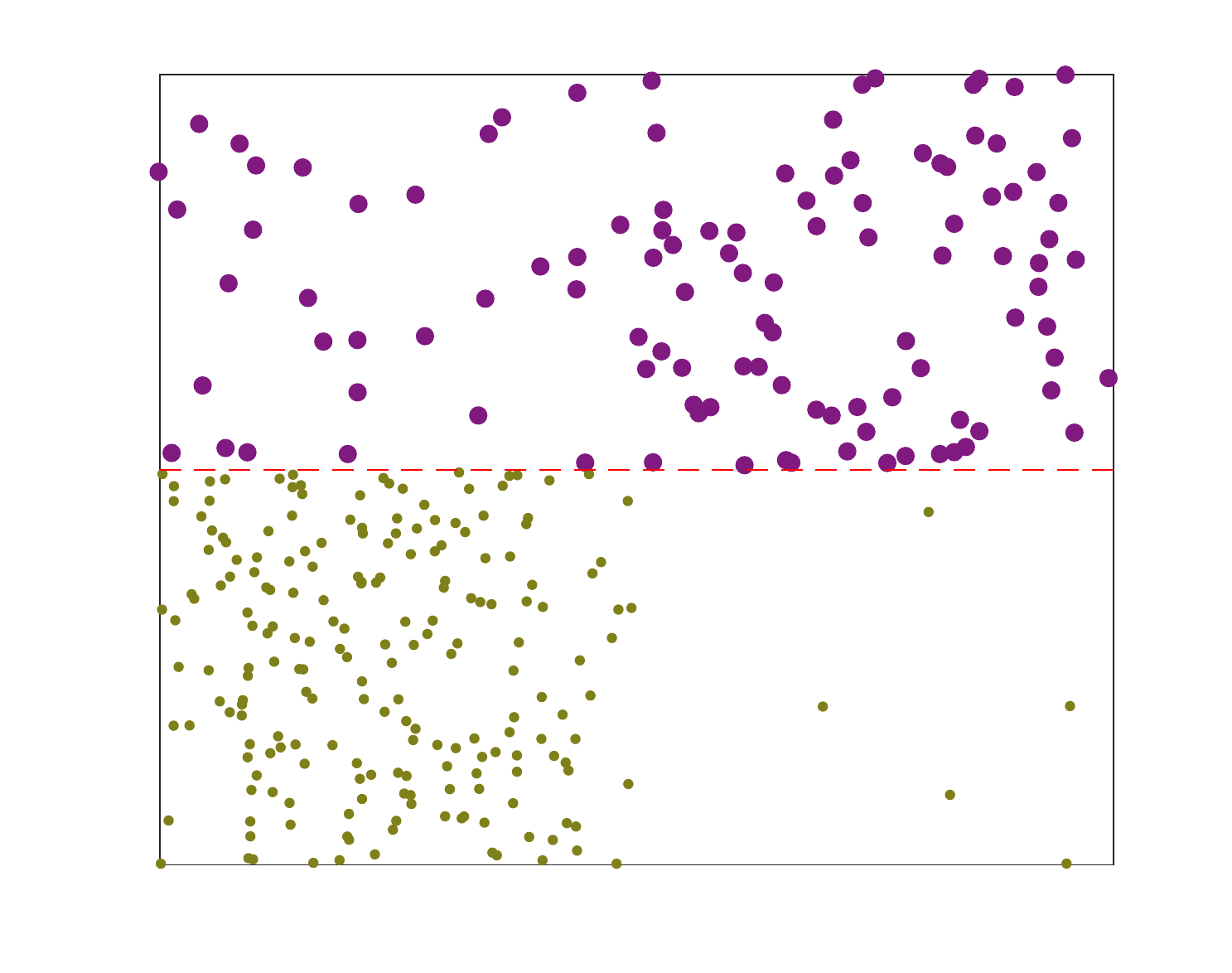}
    }
    \hspace{-20pt}
    \subfigure[Pre-clustering (K-means)]{
    \includegraphics[width=3.7cm, height=3cm]{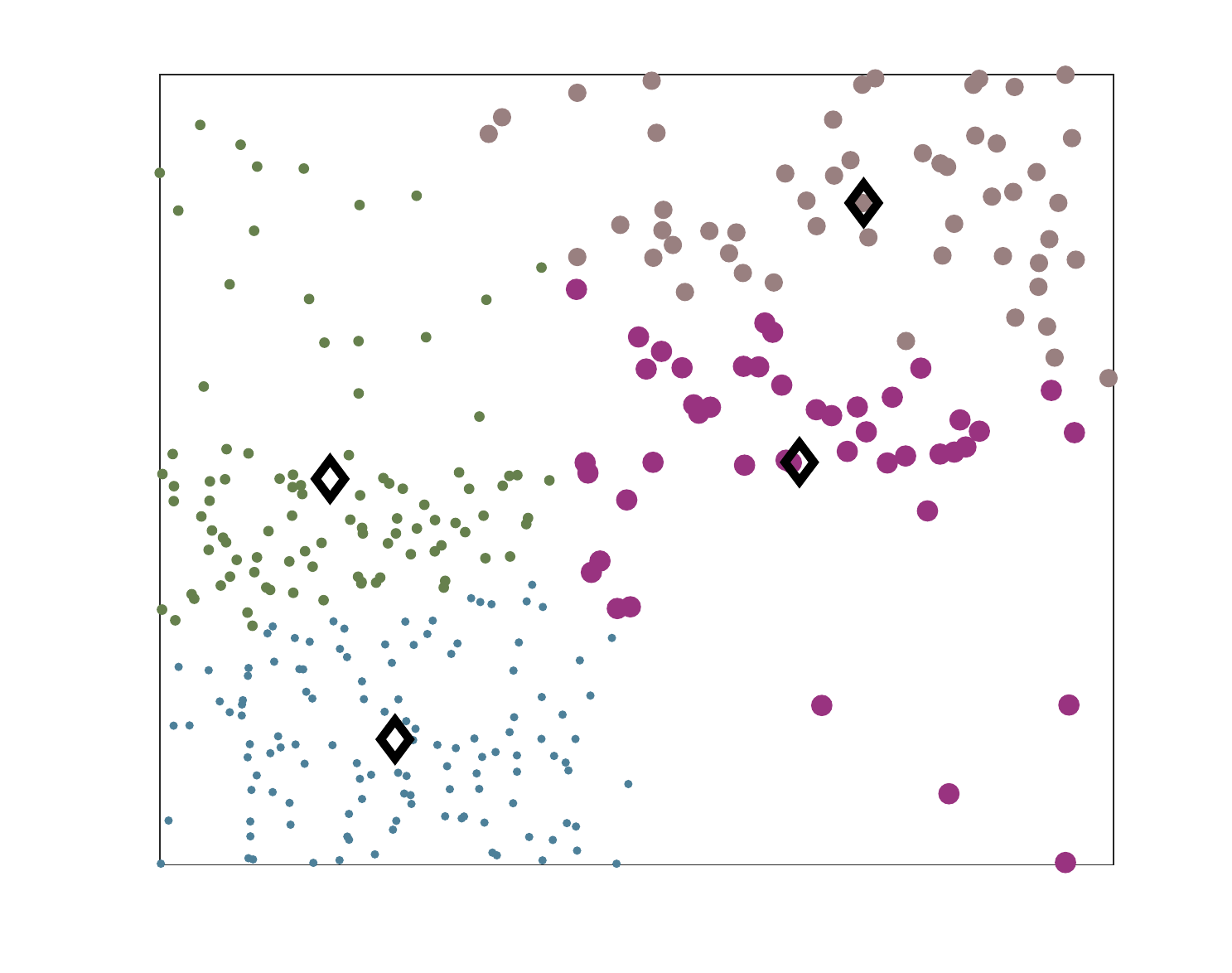}
    }
    \hspace{-20pt}
    \subfigure[$k$-DPP sampling]{
    \includegraphics[width=3.7cm, height=3cm]{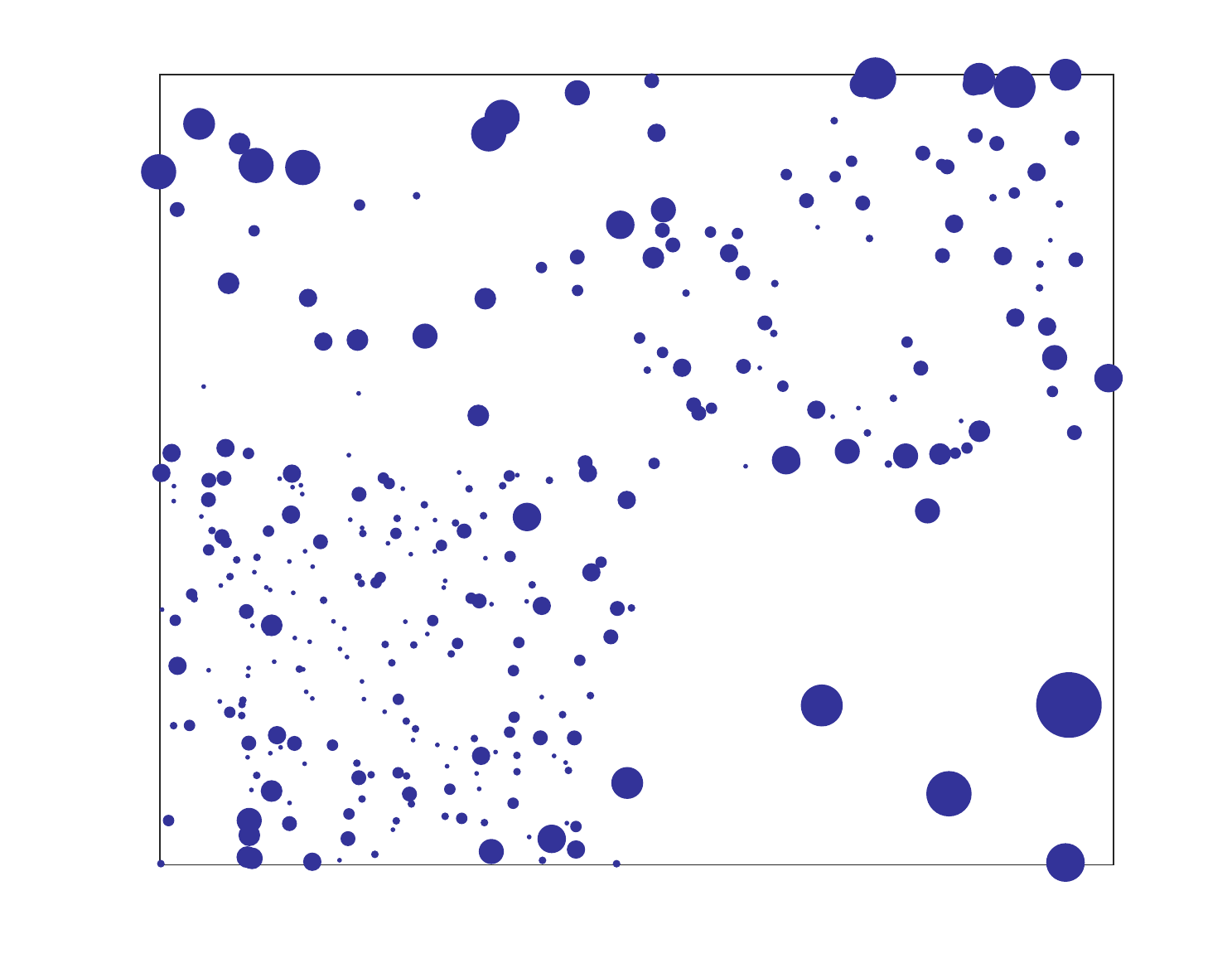}
    }
    \hspace{-20pt}
    \caption{Visualization of different non-uniform data subsampling schemes on toy data. Panel (a) shows a homogeneous distribution of data. We assume that we only observe an imbalanced subset, shown in panel (b). 
    Panels (c), (d), (e) demonstrate different biased sampling methods that aim at restoring balance in the data. Thicknesses of data points thereby indicate their sampling frequency. Biased stratified sampling (c) relies on dividing the feature space vertically along certain dimensions, whereas pre-clustering (d) defines the strata as clusters obtained from k-means~\cite{fu2017CPSGMCMC} (we used $k=4$). The black diamonds show the cluster centers and data are colored with respect to their cluster membership. 
    %It advances the traditional stratified sampling since it is a data-driving approach to find strata. However, the clustering method itself suffers from data imbalance and gives unsatisfying result. 
    Panel (e) shows the results using the $k$-DPP, using an RPF kernel of spatial distances as similarity measure between data points. In this example, the $k$-DPP best restores the balance of the original data set. %The sampling size $k=10$ is used. 
    %We can see that our method lead to population balance with respect to point relationship directly.
    }
    \label{fig:toyDataCompare}
\end{figure*}

\paragraph{Expected, empirical, and diversified risk.}~ Many problems in machine learning amount to minimizing some loss function $\ell(x,\theta)$ which both depends on a set of parameters $\theta$ and on data $x$. In probabilistic modeling, $\ell$ could be the negative logarithm of the likelihood of a probabilistic model, or a variational lower bound \cite{blei2003latent,jordan1999introduction}. We often thereby assume that the data were generated as draws from some underlying unknown data-generating distribution $p_{\text{data}}(x)$, also called the population distribution. To best generalize to unseen data, we would ideally like to minimize this function's expectation under $p_{\text{data}}$,
\begin{equation}
J(\theta) = \E_{x\sim p_\text{data}} [\ell(x;\theta)]
\end{equation}
This objective function is also called expected risk~\cite{bottou2010large}.
Since $p_{\text{data}}(x)$ is unknown and we believe that our observed data are in some sense a representative draw from the population distribution, we can replace the expectation by an expectation over the \emph{empirical} distribution of the data $p_\text{emp}$, which leads to the empirical risk~\cite{bottou2010large},
\begin{equation}
\vspace{-5pt}
\hat{J}(\theta) = \E_{x\sim p_\text{emp}} [\ell(x;\theta)] = {\frac{1}{N}}\sum_{i=1}^N \ell(x_i,\theta).
\label{eq:empirical_risk}
\vspace{-3pt}
\end{equation}
%However, in many cases, our observed data are not representative of the population. Instead,  we are observing a \emph{biased} sample, meaning that some types of the data are systematically more or less present in our sample. These types of biases occur very frequently and may arise due to the costs of data collection or other forms of sampling biases. 

A typical goal in machine learning is not to minimize the empirical risk with high accuracy, but to learn model parameters that generalize well to unseen data. For every data point in a test set, we wish our model to have high predictive accuracy. If this test set is more balanced than the training set (for instance, because it contains all classes to equal proportions in a classification setup), we would naturally like to train our model on a more balanced training set than the original one without throwing away data. In this work, we present a systematic way to achieve this goal based on biased subsampling of the training data.
 We term the collection of all samples generated from biased subsampling the \emph{balanced dataset}.

%In this work, we assume that the true population is more homogeneously distributed than our sample. This assumption forms a working hypothesis of our approach and empirically justified as shown in the experiment session. 
%Our goal is to substitute the empirical distribution of the data by a more homogeneous distribution which better approximates the underlying population. 

To this end, we introduce the \emph{diversified risk}, where we average the loss function over diversified mini-batches $\vec{x}$ of size $k$,
\vspace{-5pt}
\begin{equation}
\vspace{-5pt}
{J}^*(\theta) = {\textstyle \frac{1}{k}}\, \E_{\vec{x}\sim {\rm k-DPP}} [\ell(\vec{x};\theta)],
\label{eq:diversified_risk}
\end{equation}
%Our working hypothesis implies that this balanced population is closer to the underlying population distribution than the empirical distribution of the data, which is also sketched in Figure~\ref{fig:DataRelationship}. 
Due to the repulsive nature of $k$-DPP, similar data points are less likely to co-occur in the same draw. Thus, data points which are very different from the rest are more likely to be sampled and obtain a higher weight, as illustrated in Figure~\ref{fig:toyDataCompare} (e). 

The diversified risk depends both on the mini-batch size and on the similarity kernel $L$ of the data. A more theoretical analysis of the diversified risk is carried out in Section~\ref{sec:theory}.

\paragraph{Algorithm.}~  Our proposed algorithm directly optimizes the diversified risk in Eq.~\ref{eq:diversified_risk}. To this end, we propose SGD updates on diversified mini-batches of fixed size $k$,
\begin{equation}
\theta_{t+1} = \theta_t - \rho_t {\textstyle \frac{1}{k}}\sum_{i \in B} \nabla \ell(\theta,x_i), \quad B\sim {\rm k-DPP}.
\end{equation}
Above, $B \subset \{1,\dots,N\}$ is a collection of $k$ indices, drawn from the $k$-DPP. In every stochastic gradient step, we thus sample mini-batches from the $k$-DPP and carry out an update with decreasing learning rate $\rho_t$. 

Sampling from the $k$-DPP first requires an eigendecomposition of its kernel.
%, $L = \sum_{n=1}^N \lambda_n v_n v_n^T$. 
This decomposition can also be approximated and has to be computed only once for one dataset. Drawing a sample then has the computational complexity $\mathcal{O}(N k^3)$, where $k$ is the mini-batch size, which is much more efficient since $k$ is commonly small. This approach is briefly summarized in Algorithm \ref{alg:BP-SGD}; details on the sampling procedure are given in the supplementary material. For more details, we refer to~\cite{kulesza2012determinantal} and to~\cite{affandi2013nystrom,li2015efficient} for more efficient sampling procedures.

\begin{algorithm}[h]
% Define mini-batch size $S$ and initialize $\lambda$ randomly\\
 %Compute the similarity measure of the whole dataset $L$;\\
 %Compute the eigendecomposition $\{ (v_n, \lambda_n)\}^N _{n=1} $ of $L$;\\
{\bf Input}: Data $X$, mini-batch size $k$, eigendecomposition $\{ (v_n, \lambda_n)\}^N _{n=1} $ of similarity matrix $K$.\\
 \For{$t=0$ to $MaxIter$}{
 %%%%
  \textbf{Sample a mini-batch using the $k$-DPP}\\
   Sample $k$ eigenvectors $V$ using eigenvalues;\\
   \hspace{20pt} Sample mini-batch $\vec{x}$ of size $k$ using $V$. (See supplement.)\\
 %%%%
  \textbf{Update  parameters}\\
$\theta_{t+1} = \theta_{t} + \rho_t g^*(\theta_t; \vec{x})$ (g* is the gradient estimate)

  }
 \caption{DM-SGD}
 \label{alg:BP-SGD}
\end{algorithm}

Stochastic Variational Inference (SVI) employs SGD for training probabilistic graphical models, such as Latent Dirichlet Allocation (LDA).
%In case of probabilistic graphical models, such as Latent Dirichlet Allocation (LDA), stochastic variational inference (SVI) is needed, which is a special case of SGD. 
Every SVI update involves an inner loop.  Algorithm \ref{alg:BP-SVI-LDA} shows an application of DM-SGD to SVI for LDA~\cite{blei2003latent}. We thus term it \emph{DM-SVI}.

% Note that at first sight, the linear scaling with the number of data points $N$ seems to spoil the efficiency of SGD, which does not depend on $N$. However, the sampling procedure is independent of the machine learning model. This indicates that we can draw the samples in parallel while learning the model or even draw the samples as a pre-processing step across all models. 

% Moreover, we can benefit from this sampling scheme in different ways.  First, SVI with inner loops as in latent variable models is computationally expensive, which makes a more sophisticated choice of mini-batches desirable. 
% Second, as we show, diversified mini-batches have the benefit of variance reduction, which may also compensate for the overhead of sampling from the DPP. Third, there are approximate versions of $k$-DPP sampling which are scalable to big datasets  \cite{affandi2013nystrom,li2016fast}. 
% For example, fast $k$-DPP is used in our experiment in section \ref{sec:cnn} for a large scale dataset.

\begin{algorithm}[h]
\textrm{We adopt the notation from~\cite{hoffman2010online}.} \\
 \For{$t=0$ to $MaxIter$}{
 %%%%
  \textbf{Sample a mini-batch using the $k$-DPP};\\   
  \textbf{Update variational parameters};\\
  \For{$j=0$ to Mini-batch Size}{
  Update local variational parameters ( e.g. $\phi$ and $\lambda$ for LDA)  for mini-batch. \\
  }
  Compute the intermediate global parameters as if the mini-batch is replicated $\frac{D}{S}$ times. \\
  ( e.g. $\tilde{\lambda}_{kw} = \eta + \frac{D}{S}\sum_{s=1}^{S} n_{tw} \phi_{twk}$ for LDA)\\
  Update the current estimate of the global variational
parameters with $\rho_t = (\tau_0+t)^{-k}$. \\
$\lambda = (1-\rho_t) \lambda + \rho_t \tilde{\lambda}$
  }
 \caption{DM-SVI}
 \label{alg:BP-SVI-LDA}
 \vspace{-3pt}
\end{algorithm}

%SM: the following is unclear, can you explain this better? How do you compute the polynomials?
%We thereby follow three steps. First, compute the elementary symmetric polynomials $e$. Then, sample $k$ eigenvectors $V$ with index $J$ during the eigenvalues and $e$. Finally, sample $k$ data points indexed by Y using $V$. 

\paragraph{Variance reduction and connections to biased stratified sampling.}~
Dividing the data into different strata and sampling data from each stratum with adjusted probabilities may reduce the variance of SGD. This insight forms the basis of stratified sampling \cite{zhao2014accelerating}, and the related pre-clustering based method \cite{fu2017CPSGMCMC}. As we will demonstrate rigorously in the next section, our approach also enjoys variance reduction but does not require an artificial partition of the data into clusters.
%strata.

For many models, the gradient varies smoothly as a function of the data. Subsampling data from diversified regions in data space will therefore decorrelate the gradient contributions. This, in turn, may reduce the variance of the stochastic gradient.  
%Sampling only a sub-region of data 
%in a mini-batch may make the stochastic gradient change drastically. This induces gradient noise. It is therefore beneficial when each mini-batch contains a diversified set of data points, such that all regions of the data space are evenly covered. 
To some degree, methods such as biased stratified sampling or pre-clustering sample data from diversified regions, but ignore the fact that gradients \emph{within} clusters may still be highly correlated.  If the data are not amenable to clustering, this variance may be just as large as the inter-cluster variance. Our approach does not rely on the notion of clusters. Instead, we have a continuous measure of similarity between samples, given by the similarity kernel. This applies more broadly.

 In Figure \ref{fig:toyDataCompare}, we investigate how well our subsampling procedure using the $k$-DPP allows us to recover an original distribution of data from which we only observe an imbalanced subset. Panel (a) shows the original (uniform) distribution of data points, and (b) shows the observed data set which we use to re-estimate the original dataset. While biased stratified sampling (c) or pre-clustering based on k-means (d) need an artificial way of dividing the data into finitely many strata and re-balance their corresponding weights, our approach (e) relies on a continuous similarity measure between data and takes into account both intra-strata and inter-strata variations.
 %intra-strata and inter-strata variations. 

\paragraph{Computational overhead.}
Sampling from the $k$-DPP  implies a computational overhead over classical SGD. 
Regarding the overall runtime, the benefits of the approach therefore come mainly into play in setups where each gradient update is expensive. One example is stochastic variational inference for models with local latent variables. For example, in LDA, the computational bottleneck is to update the per-document topic proportions. The time spent on sampling a mini-batch using the $k$-DPP is only about $10\%$ of the time to infer these local variables and estimate the gradient (See Table \ref{tab:R8time} in Section \ref{sec:exp}). Spending this tiny overhead on actively selecting training examples is well invested as the resulting stochastic gradient has a lower variance.

Since the sampling procedure is independent of the learning algorithm, we can parallelize it or even draw the samples as a pre-processing step and reuse them for different hyperparameter settings.
Moreover, there are approximate versions of $k$-DPP sampling which are scalable to big datasets \cite{affandi2013nystrom,li2016fast}. In this paper, we use the \emph{fast $k$-DPP} \cite{li2016fast} in our large-scale experiments (Section \ref{sec:cnn}).

\section{THEORETICAL CONSIDERATIONS}
\label{sec:theory}
In this section, we give the theoretical foundation of the DM-SGD scheme.  
We first prove that biased stratified sampling and pre-clustering emerge as special cases of our algorithm for particular choices of the kernel matrix $L$. We then prove that the diversified risk of DM-SGD is a re-weighted variant of the empirical risk, where the weights are given by the marginal likelihoods of the $k$-DPP (we also present an unbiased DM-SGD scheme which approximates the true gradients, but which performs less favorably in practice). Last, we investigate under which circumstances DM-SGD reduces the variance of the stochastic gradient.

\paragraph{Notation.~}~~ For what follows, let $m_i \in \{0,1\}$ denote a variable which indicates whether the $i_{th}$ data point was sampled under the $k$-DPP.  Furthermore, let ${\mathbb E}[\cdot] = {\mathbb E} _{m \sim{\rm k-DPP}}[\cdot]$ always denote the expectation under the $k$-DPP. This lets us express the expectation $F(x)=\sum_i f(x_i)$ which depends additively on the data points $x_i$ as
\begin{equation}
    \E[{\textstyle \sum_{i=1}^N} m_i f(x_i)]\equiv \E_{x\sim {\rm k-DPP}}[F(x)] 
\end{equation}
Next, we introduce short hand notations for  first and second moments. Denote the marginal probability for a point $x_i$ being sampled as
\begin{equation}
\vspace{-5pt}
    b_i \equiv \E[m_i], 
    %b_i  \equiv \E _{x \sim k-\text{DPP}} \Big[ I [ x_i \text{is in the mini-batch}] \Big]
\end{equation}
which has an analytic form and can be computed efficiently. 
We also introduce the correlation matrix
\begin{equation}
    C_{ij} = \frac{\E[(m_i-b_i) (m_j-b_j)]}{\E[m_i] \E[m_j]} = \frac{\E[m_i m_j]}{b_i b_j} - 1.
    \label{eq:Correlation}
\end{equation}
In contrast to minibatch SGD where $\E[m_i m_j] = \E[m_i]\E[m_j]$ and hence $C_{ij}=0$, this is no longer true under the $k$-DPP. Instead, the correlation can be both negative (when data points are similar) and even positive (when data points are very dissimilar).

Lastly, let $g(\theta,x) = \sum_{i=1}^N g(\theta,x_i)$ denote the gradient of the empirical risk, which is the batch gradient, and $g(\theta,x_i)$ its individual contributions from the data $x_i$.

% Connection to stratification 
We first prove that our algorithm captures two important limiting cases, namely (biased) stratified sampling and pre-clustering.  
\begin{proposition}
Biased stratified sampling (StS) \cite{zhao2014accelerating}, where data from different strata are subsampled with equal probability, is equivalent to DM-SGD with a similarity matrix $L$, defined as a block-diagonal matrix with 
\begin{equation}
L_{ij} = 
\begin{cases}
1 & H_i = H_j\\    
0 & H_i \neq H_j,
\end{cases}
\end{equation}
where $H_i$ denotes the label for the stratum of data point $i$.
\label{thrm:stratified}
\end{proposition}
\begin{proof}
\vspace{-10pt}
It is enough to show that a draw $A$ from the $k$-DPP which has multiple data points with the same strata assignment has probability zero.

Let $A = a\cup \bar{a}$, where $a$ is a collection of indices which come from the same stratum, and $\bar{a}$ is its disjoint complement. Because of the block-structure of $L$, we have that
$$\det(L_A) = \det(L_a)\det(L_{\bar{a}}).$$
However, $\det(L_a) = 0$ because it is a matrix of all-ones. Therefore, $\det(L_A)=0$, and hence $A$ has zero probability under the $k$-DPP. Therefore, every draw from the $k$-DPP  with $L_{ij}$ defined as above contains at most one data point from each stratum. When $k$ is the same as the number of classes, we recover StS. If $k$ is smaller than the number of classes, we provide a direct generalization of StS. 
%, which is a version of stratified sampling.
%Let $M$ be the number of strata. Sub-sampling $k=M$ subset with $k$-DPP corresponds to draw one data point from each stratum. Subsampling multiple times $h$ with $k$-DPP corresponds to sample $h$ points from each stratum. T 
\end{proof}

\begin{proposition}
Pre-clustering \cite{fu2017CPSGMCMC} results as a special case of DM-SGD, with $L_{ij} =1$ if the data points $i$ and $j$ are assigned to the same cluster, and otherwise $L_{ij}=0$. 
%Thus, the clusters are treated as strata.
\end{proposition}

It is furthermore simple to see that regular minibatch SGD results from DM-SGD when choosing the identity kernel.

Next, we analyze the objective function of DM-SGD. We  prove that the diversified risk (Eq.~\ref{eq:diversified_risk}) is given by a re-weighted version of the empirical risk (Eq.~\ref{eq:empirical_risk}) of the data.

\begin{proposition}
The diversified risk (Eq.~\ref{eq:diversified_risk})
can be expressed as a re-weighted empirical risk with the marginal $k$-DPP weights $b_i$,
$$
J^*(\theta) = \frac{1}{k}\sum_{i=1}^N b_i \,\ell(x_i,\theta).
$$
As $b_i \rightarrow k/N$ in case of a trivial similarity kernel $L = I$, this quantity just becomes the empirical risk.
\end{proposition}
\begin{proof}
\vspace{-10pt}
We employ the indicators $m_i$ defined above:
\begin{align*}
   k\, {J}^*(\theta) &= \E_{x\sim {\rm kDPP}} [\ell(x;\theta)] 
      =  \E [\sum_{i=1}^N m_i \ell(x_i;\theta)] \\
    &  = \sum_{i=1}^N \E [ m_i ]\ell(x_i;\theta) =\sum_{i=1}^N b_i\ell(x_i;\theta).
\end{align*}
\vspace{-18pt}
\end{proof}

The following corollary allows us to construct an unbiased stochastic gradient based on
DM-SGD in case we are not interested in re-balancing the population. 

\begin{proposition}
The following SGD scheme leads to an unbiased stochastic gradient:
\begin{equation}
\theta_{t+1} = \theta_t - \rho_t \frac{1}{k}\sum_{i \in B}  \frac{1}{b_i}\nabla\ell(\theta,x_i), \quad B\sim {\rm kDPP}.
\label{eq:BM-gradient}
\end{equation}
This is a simple consequence of the identity $\E[\sum_{i=1}^N \frac{m_i}{b_i} g(\theta; x_i) ] =\sum_{i=1}^N \E[ \frac{m_i}{b_i} ] g(\theta; x_i) = g(\theta,x)$.
\label{cor:unbiased}
\end{proposition}
%Corollary \ref{cor:unbiased} is valid for BM-SGD with either $k$-DPP or DPP. 

Finally, we investigate under which circumstances the DM-SGD gradient has a lower variance compared to simple mini-batch SGD on the diversified risk. To this end, consider the gradient components $g(x_i,\theta)$, $g(x_i,\theta)$ of data points $i$ and $j$, respectively, as well as their correlation $C_{ij}$ under the $k$-DPP.  A sufficient condition for BN-SGD to reduce the variance is given as follows. 

\begin{theorem}
\label{thrm:3}
Assume that for all data points $x_i$ and $x_j$ and for all parameters $\theta$ in a region of interest, the scalar product $g(x_i,\theta)^\top g(x_j,\theta)$ is always positive (negative) whenever the correlation $C_{ij}$ is negative (positive), respectively, i.e.
\begin{align}
\label{eq:conditions_varreduction}
\forall_{i\neq j}: C_{ij} \, g(x_i,\theta)^\top g(x_j,\theta) < 0.
\end{align}
Then, DM-SGD has a lower variance than SGD.
\end{theorem}

\paragraph{Remark.\;} 
The sufficient conditions outlined in Theorem~\ref{thrm:3} are very strong, but its proof provides us with valuable insights of why variance reduction occurs.

\begin{proof}
\vspace{-10pt}
To begin with, define 
\begin{align}
 g^{F}(\theta,x) & = \textstyle{\frac{1}{k}\sum_{i=1}^N} b_i \, g(\theta,x_i), \\
  g^{*}(\theta,x) & = \textstyle{\frac{1}{k}\sum_{i=1}^N} m_i \, g(\theta,x_i), 
\end{align}
where $g^*$ is the DM-SGD gradient and $g^F = \E[g^*]$ is the full gradient of the diversified risk. 

We denote the difference between the expected and stochastic gradient as
\begin{equation}
 \Delta g = g^* -g^F = \textstyle{\frac{1}{k}\sum_{i=1}^N} (b_i-m_i) \, g(\theta,x_i),
\end{equation}
By construction, this quantity has expectation zero. We are interested in the trace of the stochastic gradient covariance,
\begin{equation}
 Var(g^*) = {\rm Tr}\left( \, Cov (g^*) \right)= \E[\Delta g^\top \Delta g].
\end{equation}
This quantity can be expressed as
\vspace{-6pt}
\begin{align*}
 Var(g^*)    =  \frac{1}{k^2}\sum_{i,j=1}^N \underbrace{\E[(m_i - b_i)(m_j - b_j)]}_{\E[m_i m_j] - b_i b_j} g(x_i,\theta)^\top g(x_j,\theta)
\end{align*}
We can furthermore compute
 \begin{align*}
 \E[m_i m_j] & =  \E[m_i^2]\delta_{ij} +  \E[m_i m_j](1- \delta_{ij})\\
      & =  \E[m_i]\delta_{ij} + (C_{ij}+1) b_i b_j (1-\delta_{ij}),
 \end{align*}
where $\delta_{ij}$ is the Kronecker symbol (we used $m_i^2 = m_i$).

Collecting all terms, the variance can be written as
\vspace{-8pt}
\begin{multline*}
 Var(g^*)  \,  = \, \frac{1}{k^2}\sum_{i=1}^N (b_i - b_i^2) \,  || g(x_i,\theta)||^2_2 \\
  + \frac{1}{k^2}\sum_{i\neq j} C_{ij} b_i b_j g(x_i,\theta)^\top g(x_j,\theta).
\end{multline*}
\vspace{-3pt}
The first term is just the variance of regular mini-batch SGD, where we sample each data point with probability proportional to $b_i$, which also optimizes the diversified risk. This term is always positive because $b_i < 1$ and thus $b_i > b_i^2$. 

The second term can be both positive and negative. By a clever choice of similarity kernel and resulting correlation function $C_{ij}$  (as defined in Eq. \ref{eq:Correlation}), the second term may therefore reduce the variance.
We immediately see that this condition exactly corresponds to Eq.~\ref{eq:conditions_varreduction}. This proves our claim.
\end{proof}
\vspace{-5pt}
\paragraph{Discussion of Theorem~\ref{thrm:3}.\;} If the similarity kernel $L$ relies on spatial distances, nearby data points $x_i$ and $x_j$ have a negative correlation $C_{ij}$ under the $k$-DPP. However, if the loss function is smooth, $g(x_i,\theta)$ and $g(x_j,\theta)$ tend to align (i.e. have a positive scalar product). Eq.~\ref{eq:conditions_varreduction} is therefore naturally satisfied for these points.  $C_{ij}$ can also be positive: since some combinations of data points are less likely to co-occur, others must be more likely to co-occur. Since these points tend to be far apart, it is reasonable to assume that their gradients show no tendency to align. It is therefore plausible to assume that for these points, Eq.~\ref{eq:conditions_varreduction} also applies\footnote{We only need to assure that the negative contributions outweigh the positive ones to see variance reduction.}.

To summarize, if the condition in Eq.~\ref{eq:conditions_varreduction} is met, we can guarantee variance reduction relative to mini-batch SGD, and we have given arguments why it is plausible that these are met to some degree when using DM-SGD with a distance-dependent similarity kernel. In our experimental section we show that DM-SGD has a faster learning curve, which we attribute to this phenomenon.
\section{EXPERIMENTS}
\label{sec:exp}
We evaluate the performance of our method in different settings.  In Section \ref{sec:lda} we demonstrate the usage of DM-SGD for Latent Dirichlet Allocation (LDA) \cite{blei2003latent}, an unsupervised probabilistic topic model. We show that the learned diversified topic representations are better suited for subsequent text classification. In Section \ref{sec:softmax} we evaluate the supervised scenario based on multinomial (softmax) logistic regression with imbalanced data. We compare against stratified sampling, which emerges naturally in this example. In section \ref{sec:cnn} we show that our method also maintains performance on the balanced MNIST data set, where we tested convolutional neural networks.
%, and where we used fast $k$-DPP's~\cite{li2016fast} for better scalability. %Further results are shown in the supplementary material. 
In all the experiments, we pre-sample the mini-batch indices using the $k$-DPP implementation from \cite{kulesza2012determinantal} for small datasets, and from \cite{li2016fast} for big datasets. In this way, sampling is treated as a pre-scheduling step and can easily be parallelized.
We found that our approach finds more diversified feature representations (in unsupervised setups) and higher predictive accuracies (in supervised setups). We also found that the $k$-DPP converges within fewer passes through the data compared to standard minibatch sampling due to variance reduction.

%%%%%%%%%%%%%%%%%%%%%%%%%%%%%%%%%%%%%%%%%%%%%%%%%%%%%%%%%%%%%%%%%%%%%%%%%%%%%%%%%%%%%%%%%%%%%%%%%%%%%%%%%%%%%%%%%%%%%%%%%%%%%%%%%%%%%%%%%%%%%%%%%%%%%%
%\subsection{Topic learning with LDA}
\subsection{TOPIC LEARNING WITH LDA}
\label{sec:lda}
We follow Algorithm \ref{alg:BP-SVI-LDA} for LDA.   
Firstly, we demonstrate the performance of DM-SVI on synthetic data with LDA. We show that by balancing our mini-batches, we find a much better recovery of the  topics that were used to generate the data.
Second, we use a real-world news dataset. We demonstrate that we can learn more diverse topics that are also better features for text classification tasks.

In this setting, stratified sampling is not applicable since there is no discrete feature such as a class label available. With only word frequencies available, no simple criterion can be used to divide the data into meaningful strata. 

\subsubsection{SYNTHETIC DATA}
We generate a synthetic dataset (shown in the supplementary material) following the generative process of LDA with a fixed global latent parameter (the graphical topics). 
We choose distinct patterns as shown in Figure \ref{fig:syn_topic} (a), where each row represents a topic and each column represents a word. 
To generate an imbalanced data set, we use different Dirichlet priors for the per document topic distribution $\theta$. 300 documents are generated with prior (0.5 0.5 0.01 0.01 0.01); 50 with prior (0.01 0.5 0.5 0.5 0.01) and 10 with prior (0.01 0.01 0.01 0.5 0.5). Hence, the first two topics are used very often in the corpus. Topic 3 and 4 are shown a few times and topic 5 appears very rarely. 

\begin{figure}[t]
\begin{center}
\subfigure[Ground Truth]{
\includegraphics[width=2.51cm]{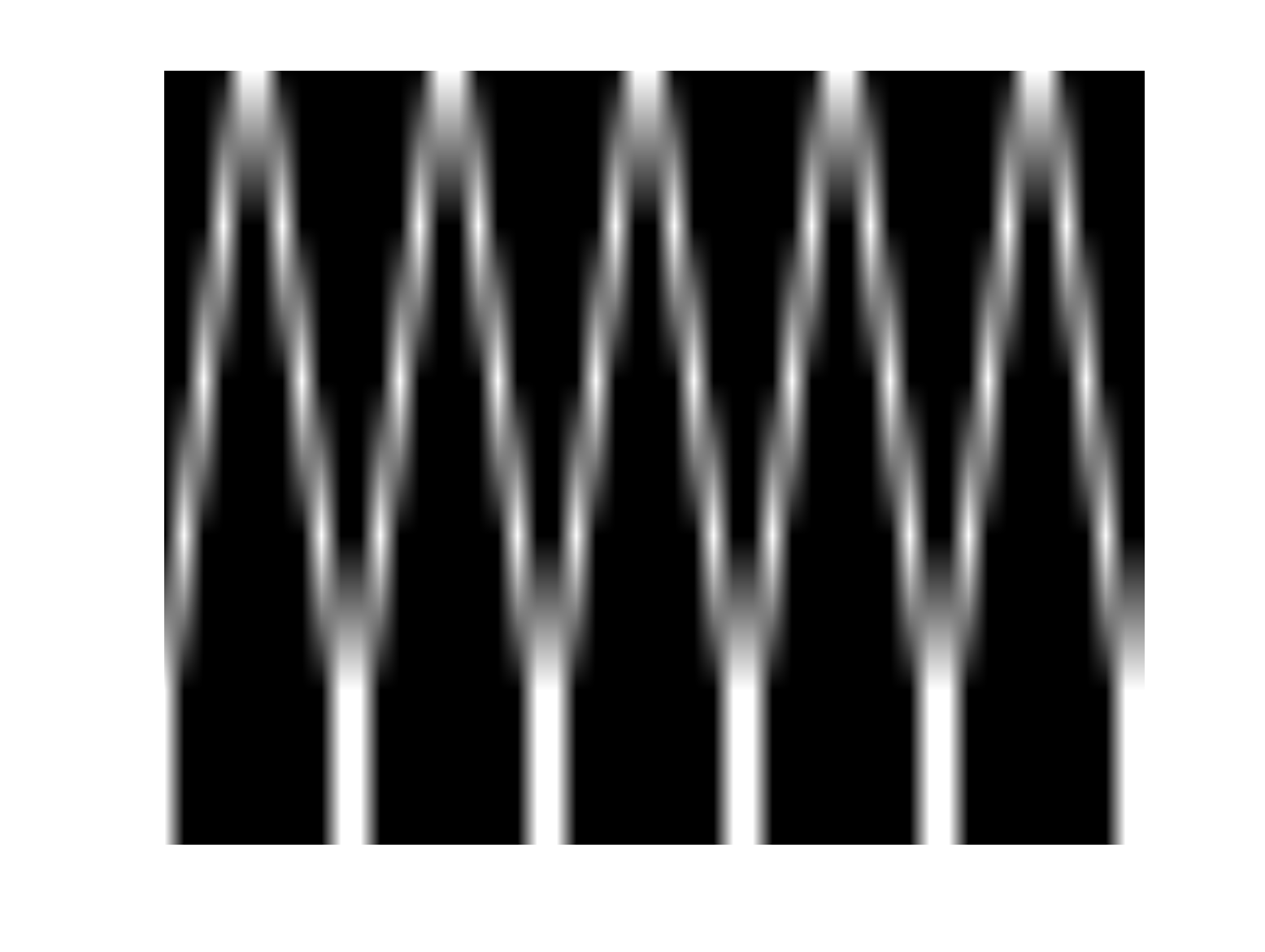}
}
\subfigure[Est w. SVI LDA]{
\includegraphics[width=2.51cm]{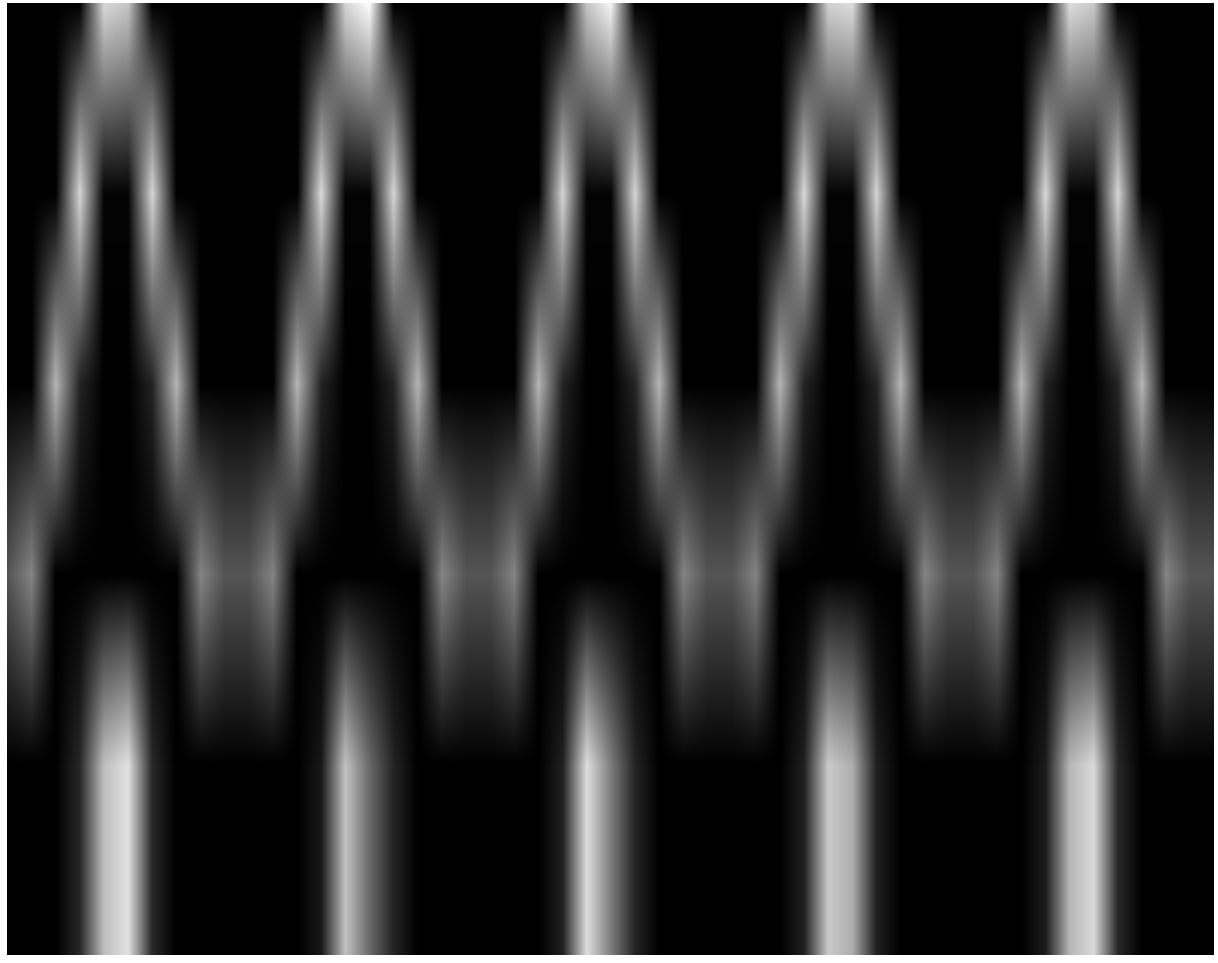}
}
\subfigure[Est w. \textbf{DM-SVI}]{
\includegraphics[width=2.51cm]{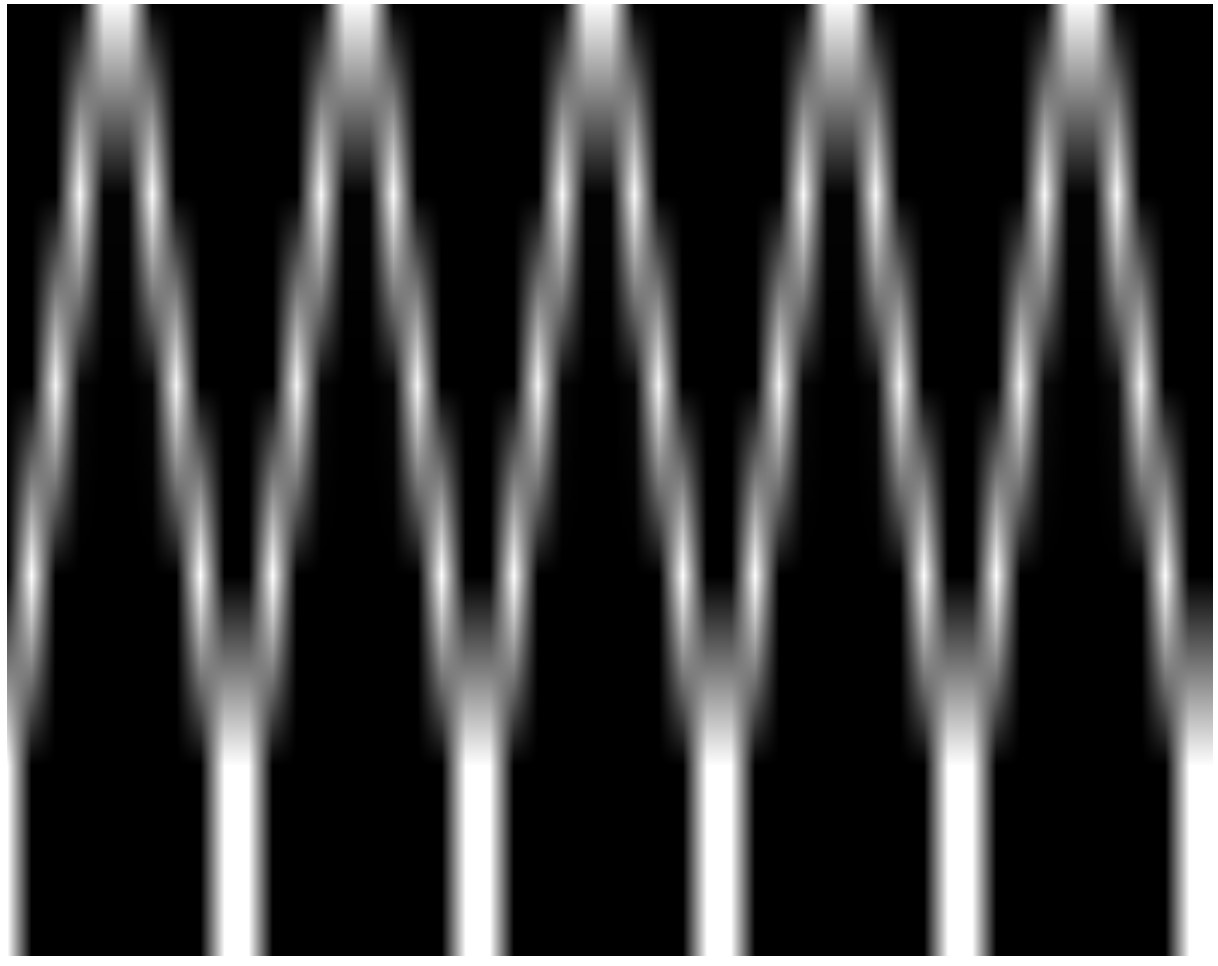}
}
\end{center}
\caption{Per topic word distribution for the synthetic data. Each row presents a topic and each column presents a word. (a) shows the ground truth with which the synthetic data is generated using LDA. (b) shows the estimation of this latent variable with LDA using traditional stochastic variational inference (SVI). (c) shows the estimation of this latent variable with DM-SVI }
\label{fig:syn_topic}
\end{figure}

We fit LDA to recover the topics of the synthetic data using traditional SVI and our proposed DM-SVI respectively. Here, the raw data occurence $x$ is used to construct the similarity matrix $L=xx^T$. We check how well the global parameters are recovered. Fully recovered latent variables indicate that the model is able to capture its underlying structure of the data. Figure   \ref{fig:syn_topic} (b) shows the estimated per topic words distribution with SVI and Figure \ref{fig:syn_topic} (c) shows the result with our proposed DM-SVI. 

In Figure \ref{fig:syn_topic} (b), we see that the first three topics are recovered using traditional SVI. Topic four is roughly recovered but with information from topic five mixed in. The last topic is not recovered at all, instead, it is a repetition of the first topic. This shows the drawback of the traditional method: when the data is not balanced, the model creates redundant topics to refine the likelihood of the dense data  but ignores the scarce data even when they carry important information. In Figure \ref{fig:syn_topic} (c),  we see that all the topics are correctly recovered thanks to the balanced dataset.
%%%%%%%%%%%%%%%%%
\subsubsection{R8 NEWS DATA EXPERIMENT}
We also evaluate the effect of DM-SVI on the Reuters news R8 dataset \cite{R8dataset}. 
This dataset contains eight classes of news with an extremely imbalanced number of documents per class, as shown in Figure \ref{fig:R8Res_hist}  (a).  To measure similarities between documents, we represent each document with a vector $x$ of the tf-idf \cite{robertson2004understanding} scores of each word in the document. Then define an annealed linear kernel
$L(x_i,x_j) = x_i^{\rho}x_j^{\rho}$ with parameter $\rho=0.1$, which is more sensitive to small feature overlap. 
We run LDA with SVI  and DM-SVI with  one effective pass through the data, where we set the mini-batch size to $80$ and use $30$ topics. %, which leads 69 iterations of updates. 
%To compare with the previous results, we ran the experiment with the same number of iterations using BM-SVI to sample these 80 documents for each mini-batch. 
%The number of topics was $30$ for each experiments. 
\begin{figure}[t]
\begin{center}
% \subfigure[Original data]{
% \includegraphics[width=4.01cm, height=3.2cm]{Pics/R8_hist_Org_new.pdf}
% }
% \hspace{-15pt}
% \subfigure[Balanced data]{
% \includegraphics[width=4.01cm, height=3.2cm]{Pics/R8_hist_DPP_new.pdf}
% }
\includegraphics[width=7cm, height=3.2cm]{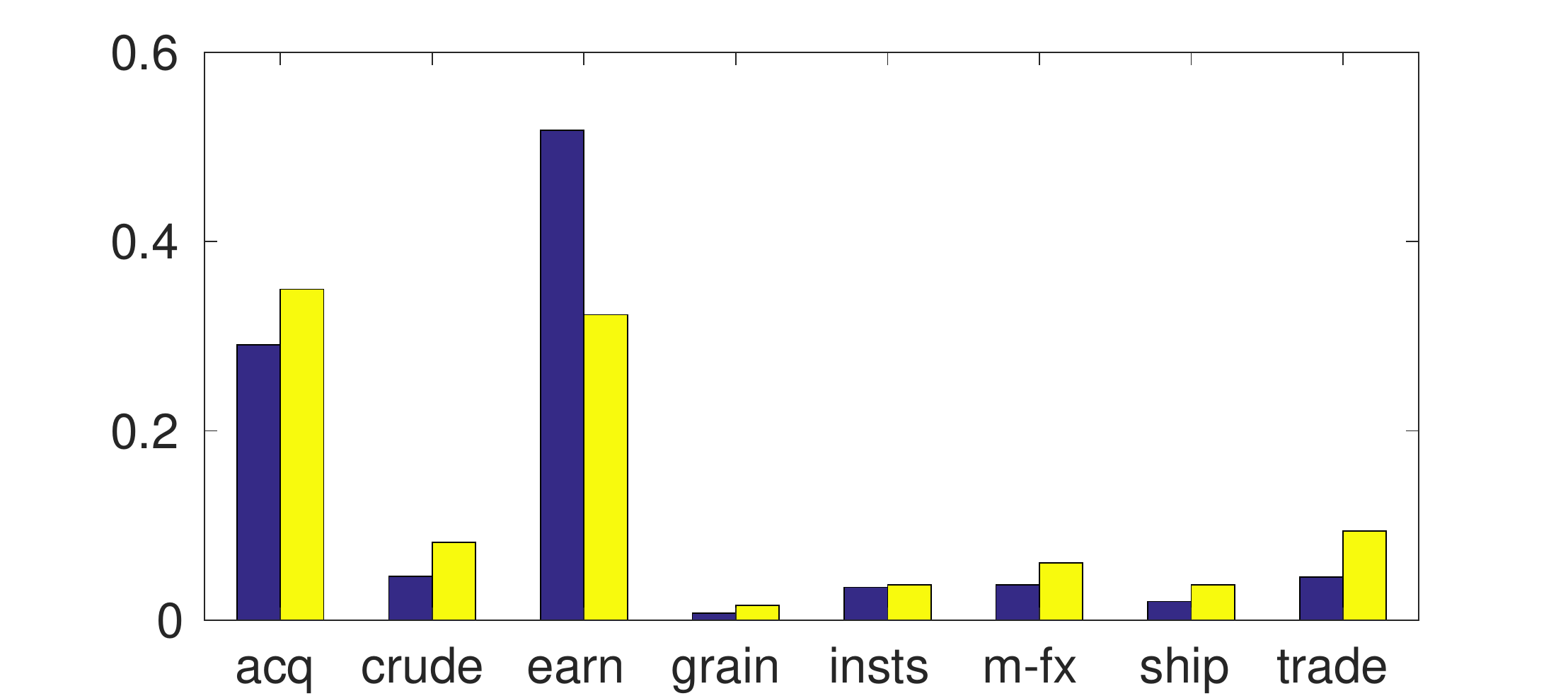}
\end{center}
\caption{
 The frequency of class labels of the training dataset (in blue) and of the balanced dataset (in yellow). While explicit class label information is withheld, the algorithm partially balances class contributions.}
\label{fig:R8Res_hist}
\end{figure}

\begin{figure}[h]
\begin{center}
\hspace{-15pt}
\subfigure[CM w.SVI]{
\includegraphics[width=4.01cm]{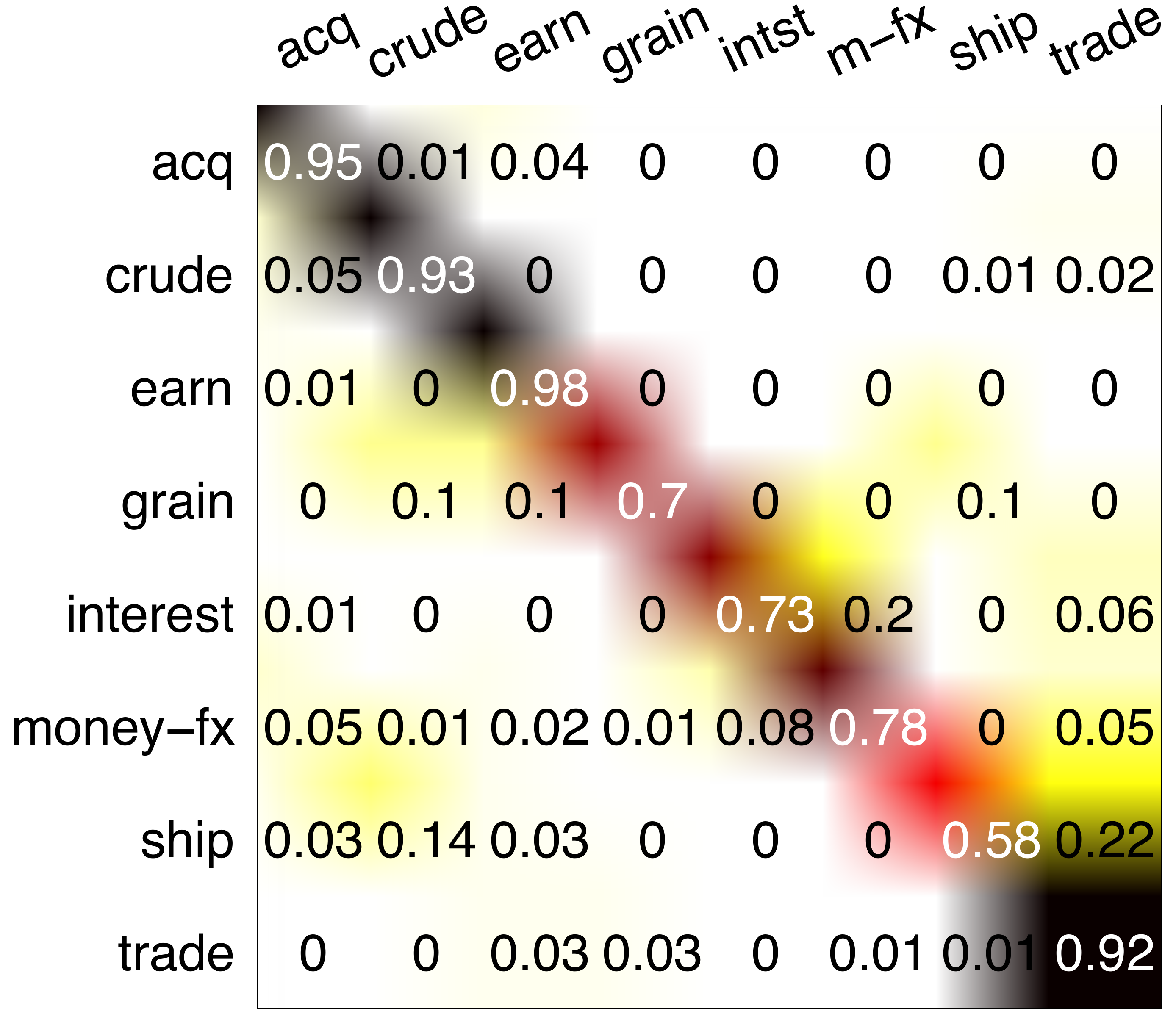}
}
\hspace{-5pt}
\subfigure[CM w.\textbf{DM-SVI}. ]{
\includegraphics[width=4.01cm]{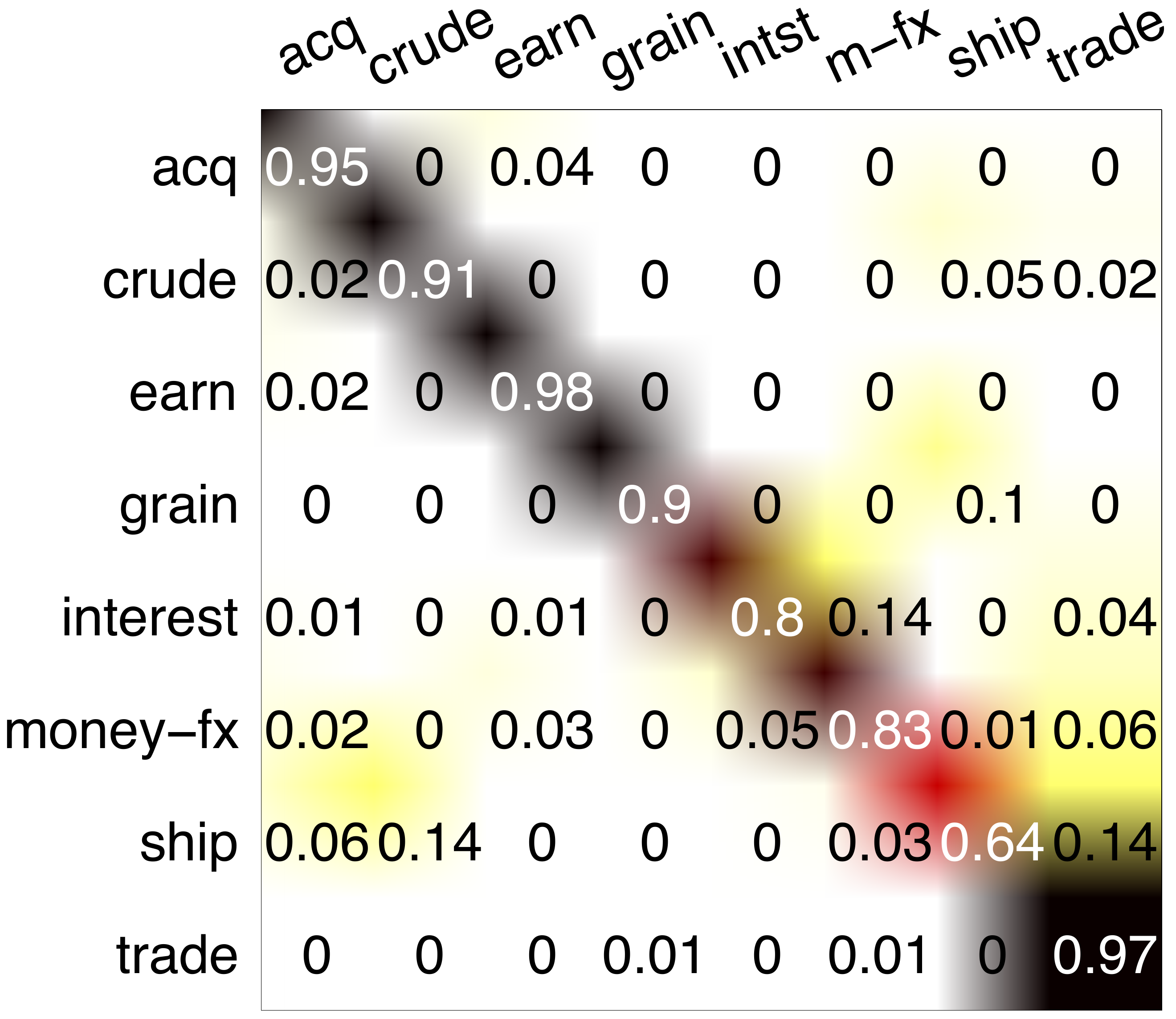}
}
\end{center}
\caption{Confusion matrix for text classification based on LDA features obtained from SVI (a) and the proposed DM-SVI (b). DM-SVI features lead to better accuracies.}
\label{fig:R8Res_CM}
\end{figure}
\begin{figure}[t]
\begin{center}
\hspace{-5pt}
\subfigure[ First 2 PC w. SVI]{
\includegraphics[width=3.9cm, height=2.5cm]{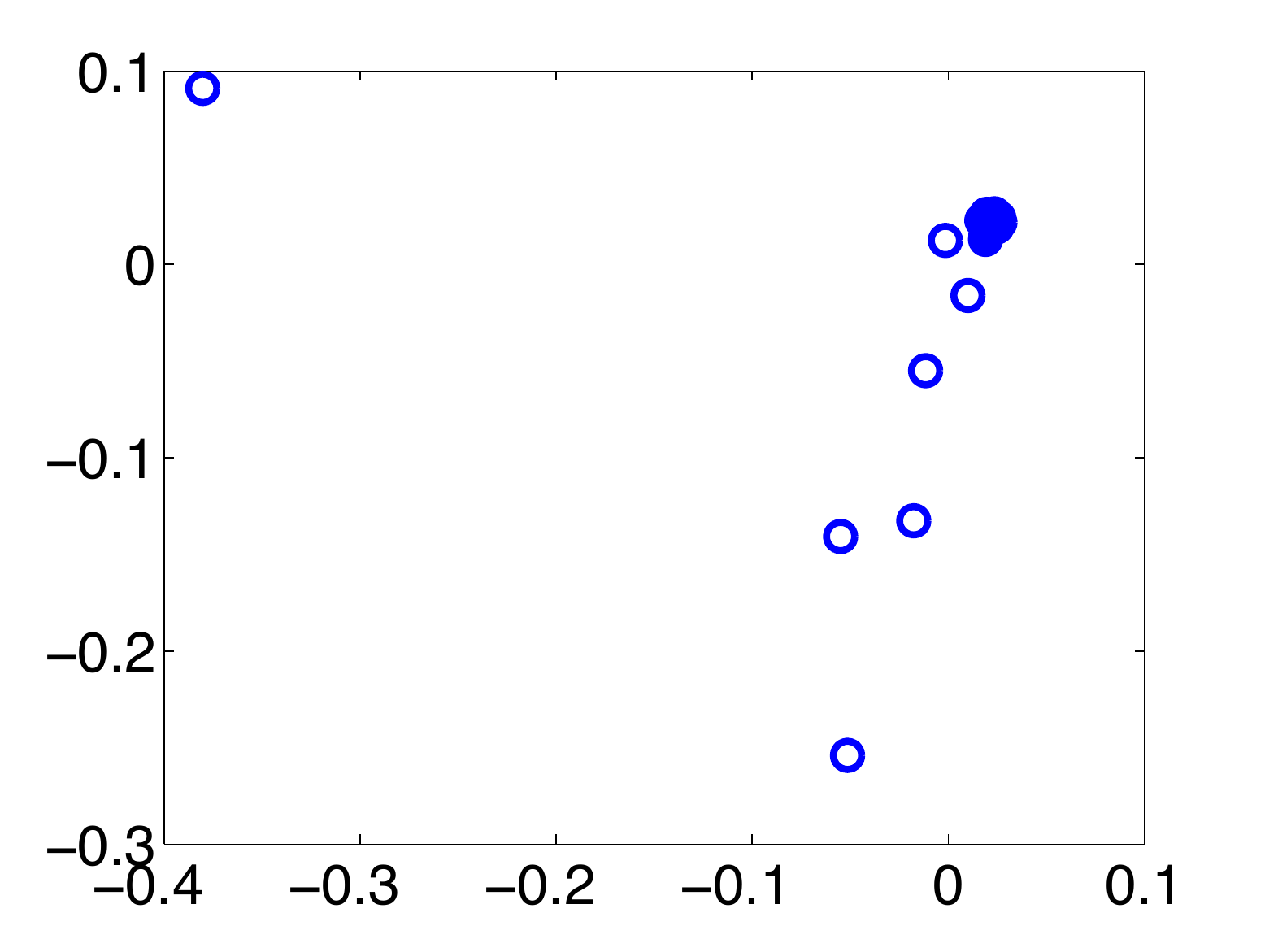}
}
\hspace{-15pt}
\subfigure[ First 2 PC w. \textbf{DM-SVI}]{
\includegraphics[width=3.9cm, height=2.5cm]{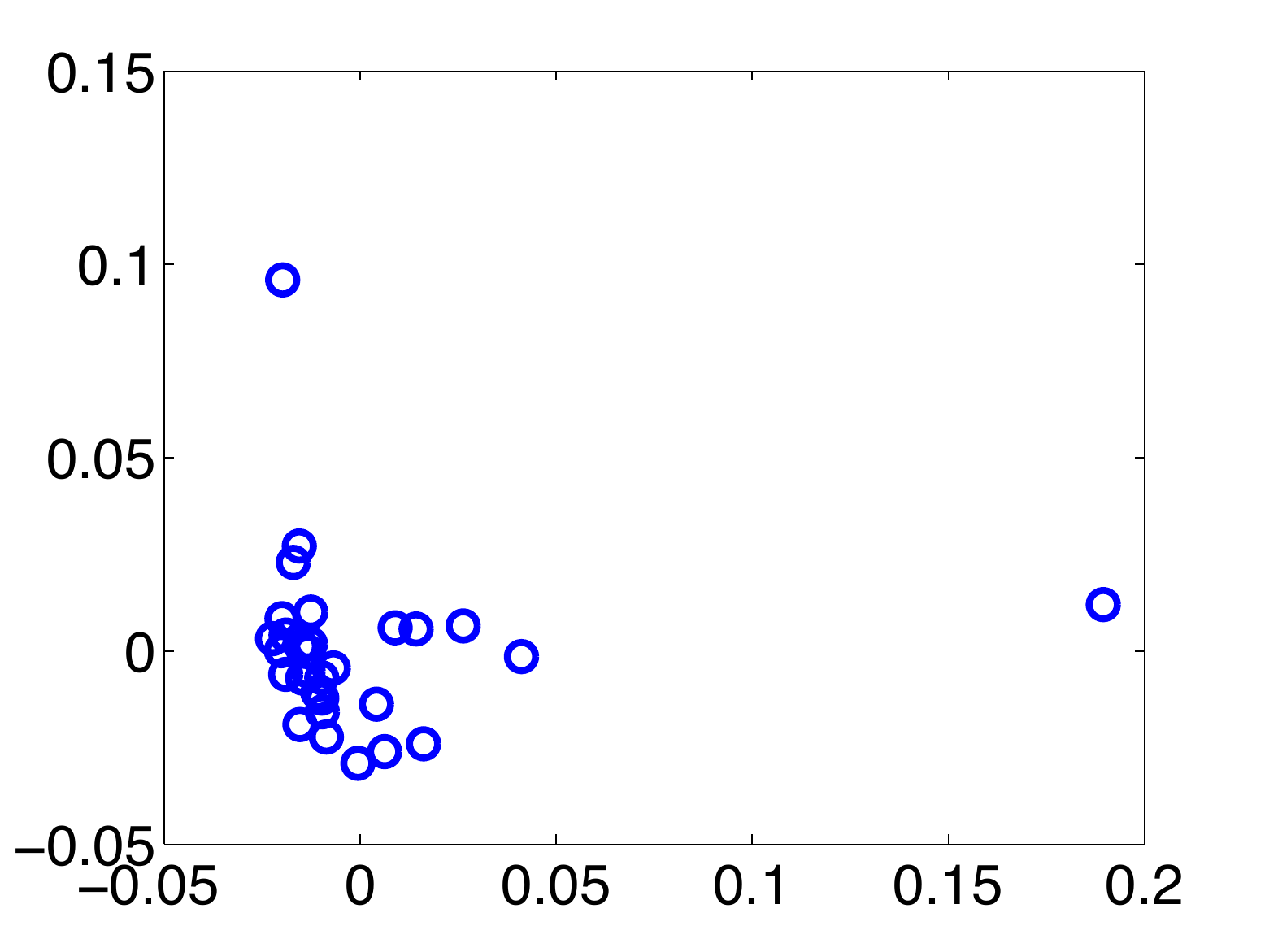}
}
\hspace{-15pt}
\end{center}
\caption{First 2 principle components of topic word distributions. DM-SVI topic vectors (b) are more diverse compared to SVI (a).}
\label{fig:R8Res_PCA}
\end{figure}
%%%%%%%%%%%%%%%%%%%%

We first compare the frequencies at which documents with particular labels were sub-sampled.  While Figure \ref{fig:R8Res_hist} shows the actual frequency of these classes in the original data set compared with the frequency of these classes over  the balanced dataset (a collection of sampled mini-batches using the $k$-DPP). We can see that the number of documents is more balanced among different classes.
%However, there are still quite big difference among numbers of documents of each class. We would like to emphasise that this is an advantage of our method. Because the goal is try to balance the population from different data clusters instead of classes. There exists sub-clusters of documents within the same class and it is desired that these clusters are balanced as well. Even in the scenario where the class label is available, the naive way of replicating documents based on the volume per class does not produce representative balanced data.  

To demonstrate that DM-SVI leads to a more useful topic representation, we classify each document in the test-set based on the learned topic proportions with a linear SVM. The global variable (per-topic word distribution) is only trained on the training set. The resulting confusion matrices  are shown in  Figure \ref{fig:R8Res_CM} using traditional SVI and DM-SVI respectively. 
With traditional SVI, the average performance over 8 classes is $82.11\%$; the total accuracy (number of correctly classified documents over number of test documents) is $94.11\%$. With DM-SVI, the average performance over 8 classes is $87.24\%$ and the total accuracy is $94.7 \%$.

Thus the overall classification performance is improved using DM-SVI features, and especially the performance on the classes with few documents (such as "grain" and "ship") is improved significantly. 

We also visualize the first two principal components (PC) of the the global topics in Figure \ref{fig:R8Res_PCA}. 
In traditional SVI, many topics are redundant and share large parts of their vocabulary, resulting in a single dense cluster. In contrast, we see that the topics in DM-SVI are more spread out. 
In this regard, DM-SVI achieves a similar effect as 
when using diversity priors as in \cite{kwok2012priors} without the need to grow the prior with the data. The top words from each topic are shown in the appendix,
where we present more evidence that the topics learned by DM-SVI are more diverse.

\begin{table}
\vspace{-10pt}
%\color{red}
\begin{tabular}{ l  c c c c} 
 \toprule
 Size  & k = 10 & k =30 & k=50	& k=80  \\
 \cmidrule(lr){1-5}
Relative cost & $0.114\%$ & $1.097\%$ & $3.191\%$ & $8.971\%$ \\ 
 \bottomrule
\end{tabular}
\medskip
\caption{\footnotesize LDA on the R8 dataset. Relative cost of mini-batch sampling as a fraction of the cost of a gradient update. Different values of mini-batch size $k$ are shown.}
\label{tab:R8time}
\end{table}

The relative costs of sampling per iteration for LDA is shown in Table \ref{tab:R8time}. Because every local update is expensive, the relative overhead of mini-batch sampling is small. More details are given in the appendix.
%%%%%%%%%%%%%%%%%%%%%%%%%%%%%%%%%%%%%%%%%%%%%%%%%%%%%%%%%%%%%%%%%%%%%%%%%%%%%%%%%%%%%%%%%%%%%%%%%%%%%%%%%%%%%%%%%%%%%%%%%%%%%%%%%%%%%%%%%%%%%%%%%%%%%%
\subsection{MULTICLASS LOGISTIC REGRESSION}
%\subsection{Fine-grained Classification with Softmax Classifier}
\label{sec:softmax}
% {\color{red} TODO discuss mini-batch size and benifit. Mini-batch size vs gradients variance trade-offs. For SGD, small mini-batch is desired
% due to the 1/sqrt(n) benefit growth.
% However, we still need a mini-batch to ensure that the gradient variance is not too big. With DPP sampling, we can take advantage of small-mini-batch more since our small-minibatch are diversified.  }

In this section, we demonstrate DM-SGD on a fine-grained classification task. The Oxford 102 flower dataset \cite{Nilsback08,sharif2014cnn} is used here for evaluation.

Many datasets in computer vision are balanced even though the true collected dataset is extremely imbalanced. The true reason is that the performance of machine learning models usually suffer from imbalanced training data.  One example is the Oxford 102 flower dataset which contains 1020 images in the training set with 10 images per class.  However, in the test set, 6149 images are available with high imbalance. In this experiment, we make the learning task harder. We use the original testing set for training and use the original training set for testing.  This setting demonstrates the real life scenario where we only can collect data with bias but wish the model to perform well in all different situations.

\begin{figure}[t]
\begin{center}
\hspace{-5pt}
\subfigure[Labels]{
\includegraphics[width=2.61cm, height =2.1cm]{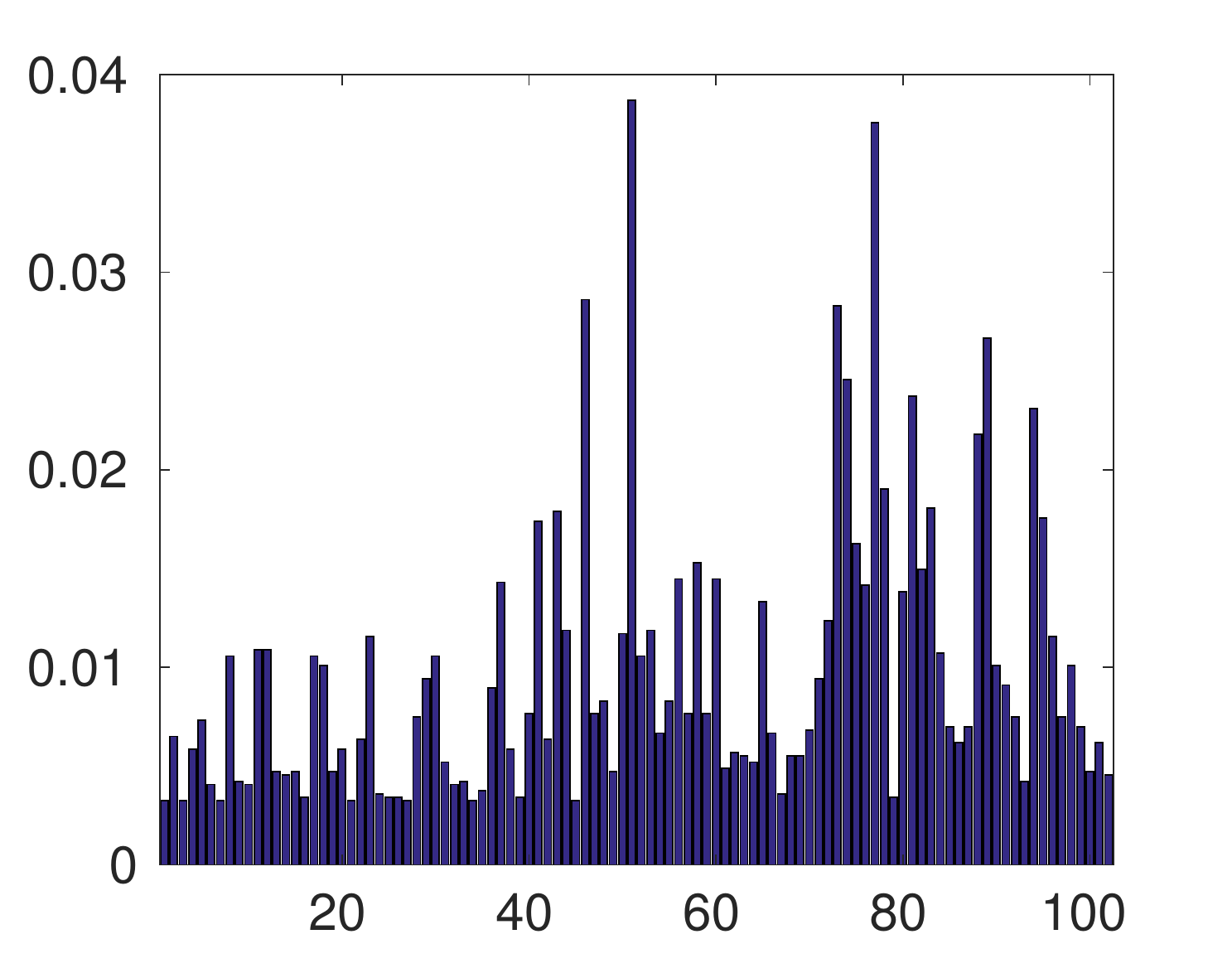}
}
\hspace{-15pt}
\subfigure[$w=0.3$]{
\includegraphics[width=2.61cm, height =2.1cm]{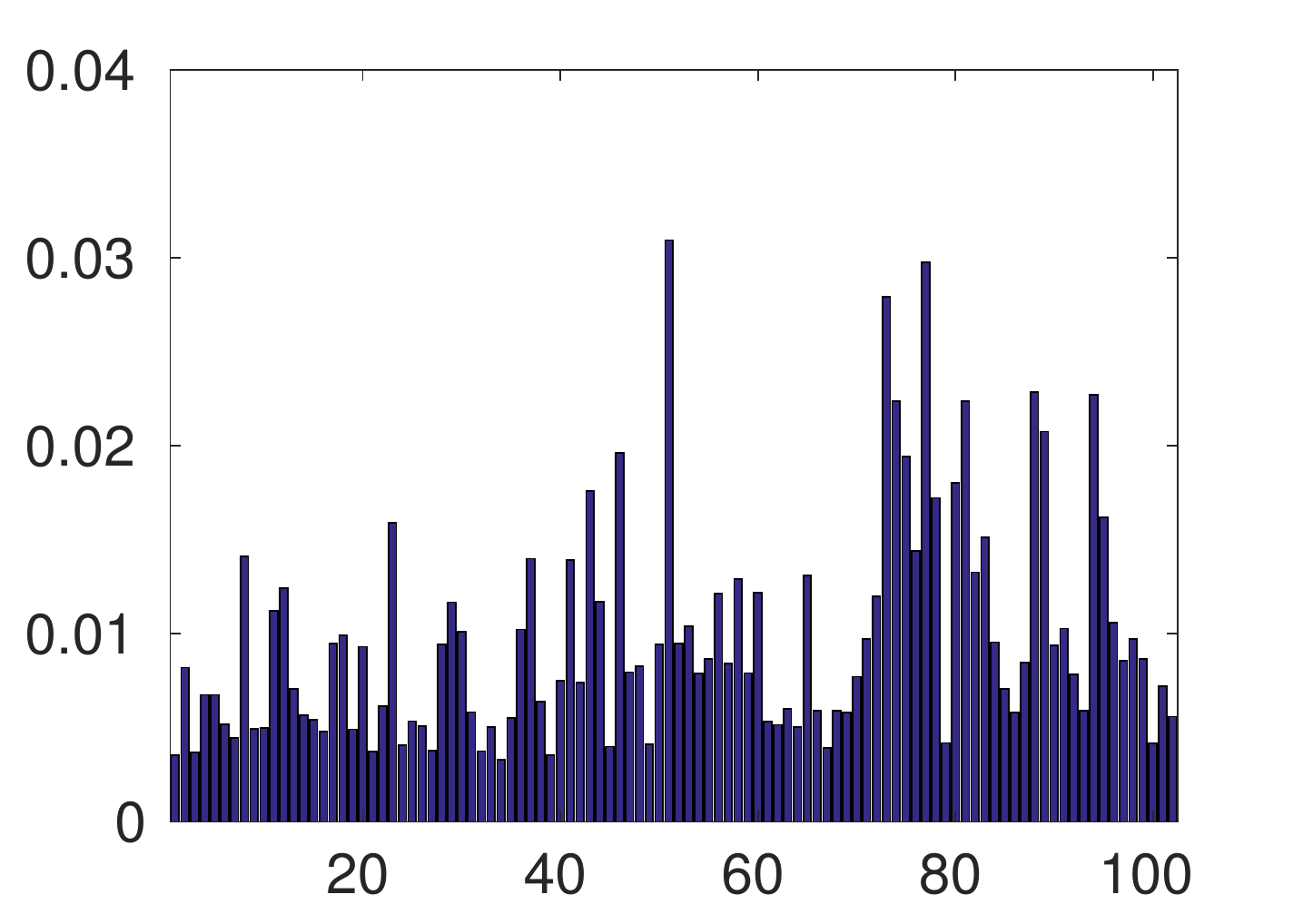}
}
\hspace{-15pt}
\subfigure[$w=0.9$]{
\includegraphics[width=2.61cm, height =2.1cm]{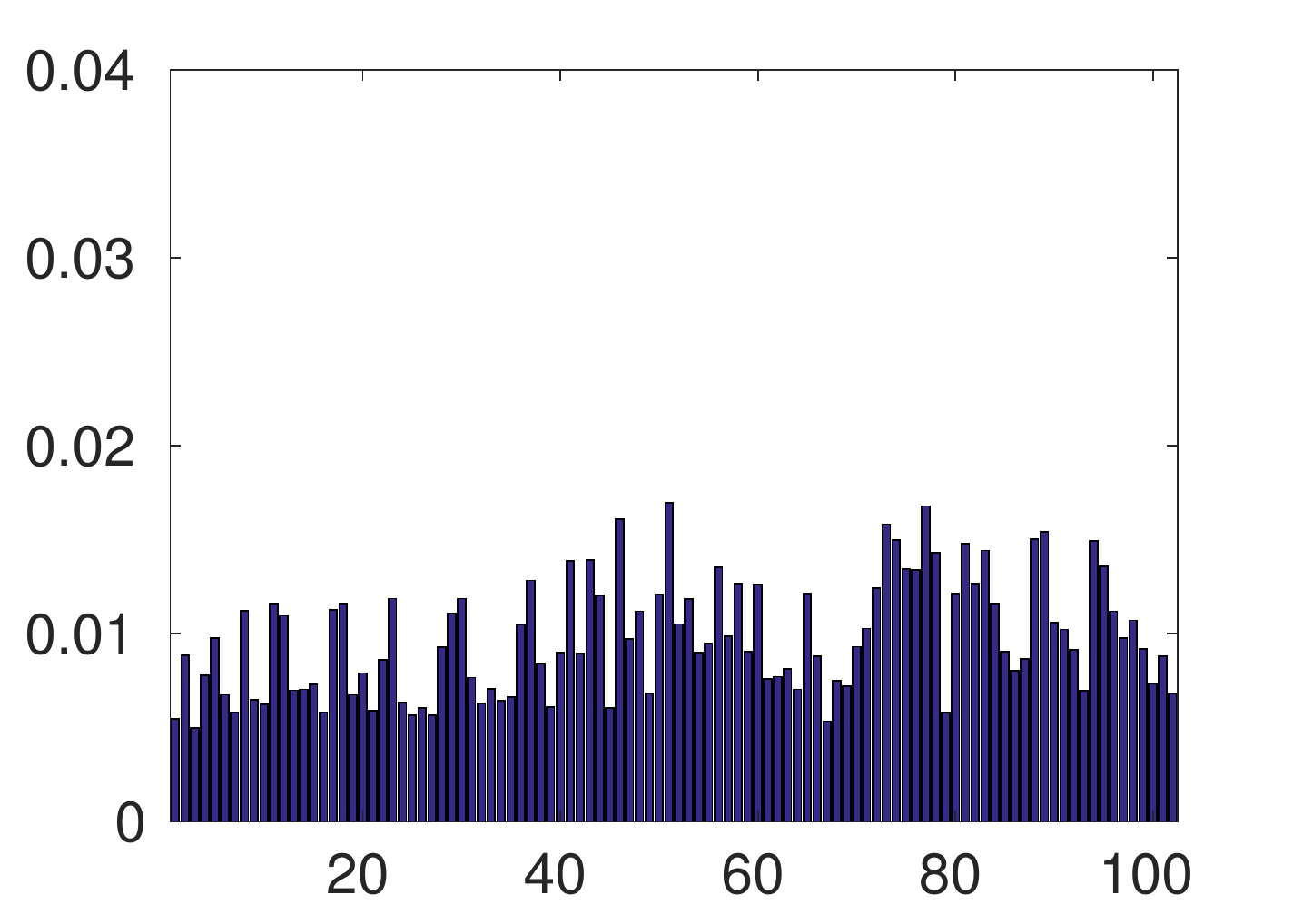}
}
\hspace{-5pt}
\end{center}
\caption{ The frequency  of data in each class of the training dataset and  of the balanced dataset using different weights $w$ (Eq.~\ref{eq:definition_w}). There are 102 classes of flowers in total and each bar presents the percentage of data belongs to one class. Minibatch size $50$ is used as an example here for (b) and (c).}
\label{fig:hist_flower} 
\end{figure}

\begin{figure}[t]
\begin{center}
\subfigure[k=50, Top3: 0.9, 0.7, 1; \protect\newline 
Best: $86.7\%$  Baseline:$84.7\%$]{
\includegraphics[width=4.cm, height = 2.85cm]{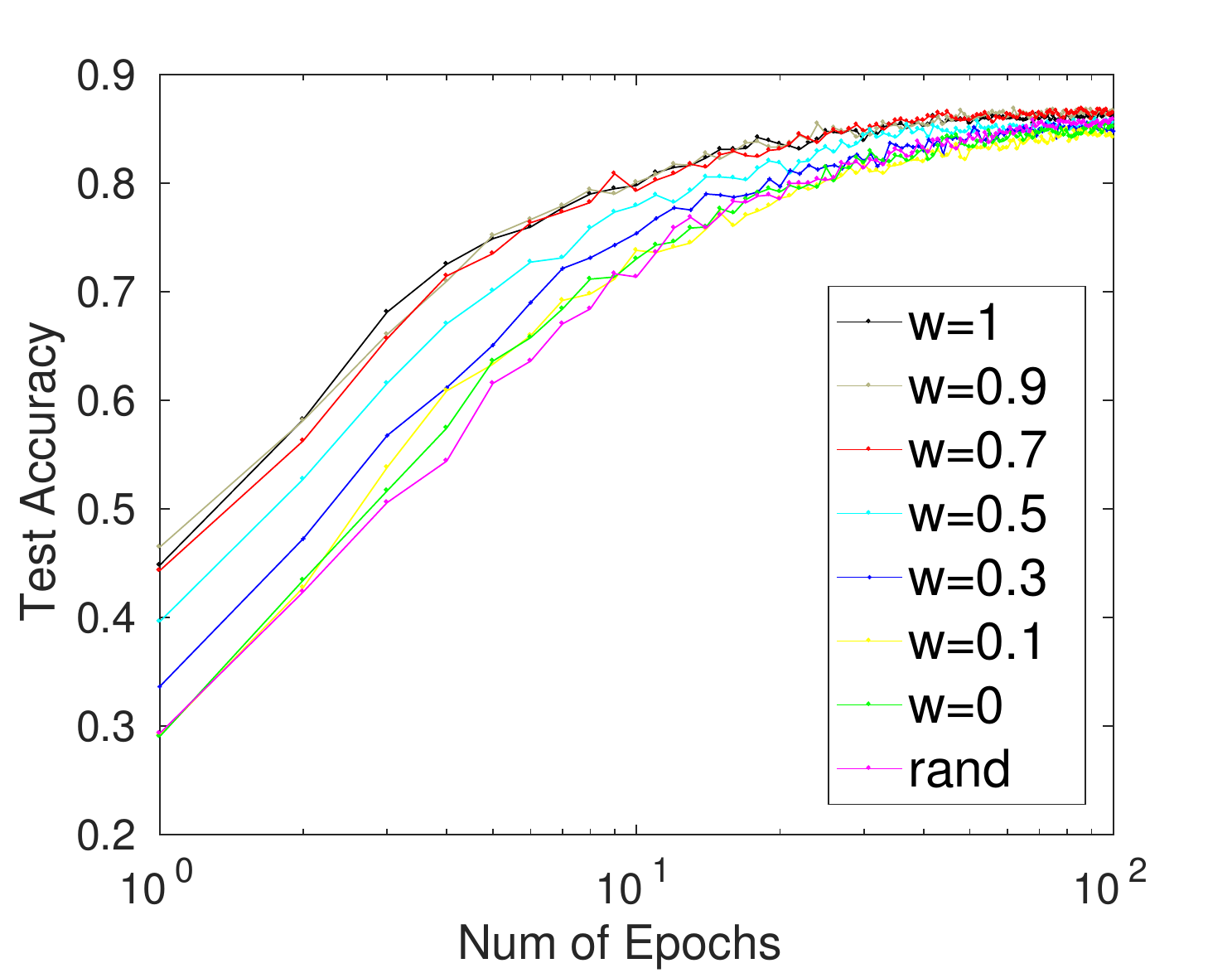}
}
\hspace{-10pt}
\subfigure[ k=80, Top3: 0.9, 0.7, 1;  \protect\newline
Best: $86.7\%$ Baseline:$81.8\%$]{
\includegraphics[width=4.cm, height = 2.85cm]{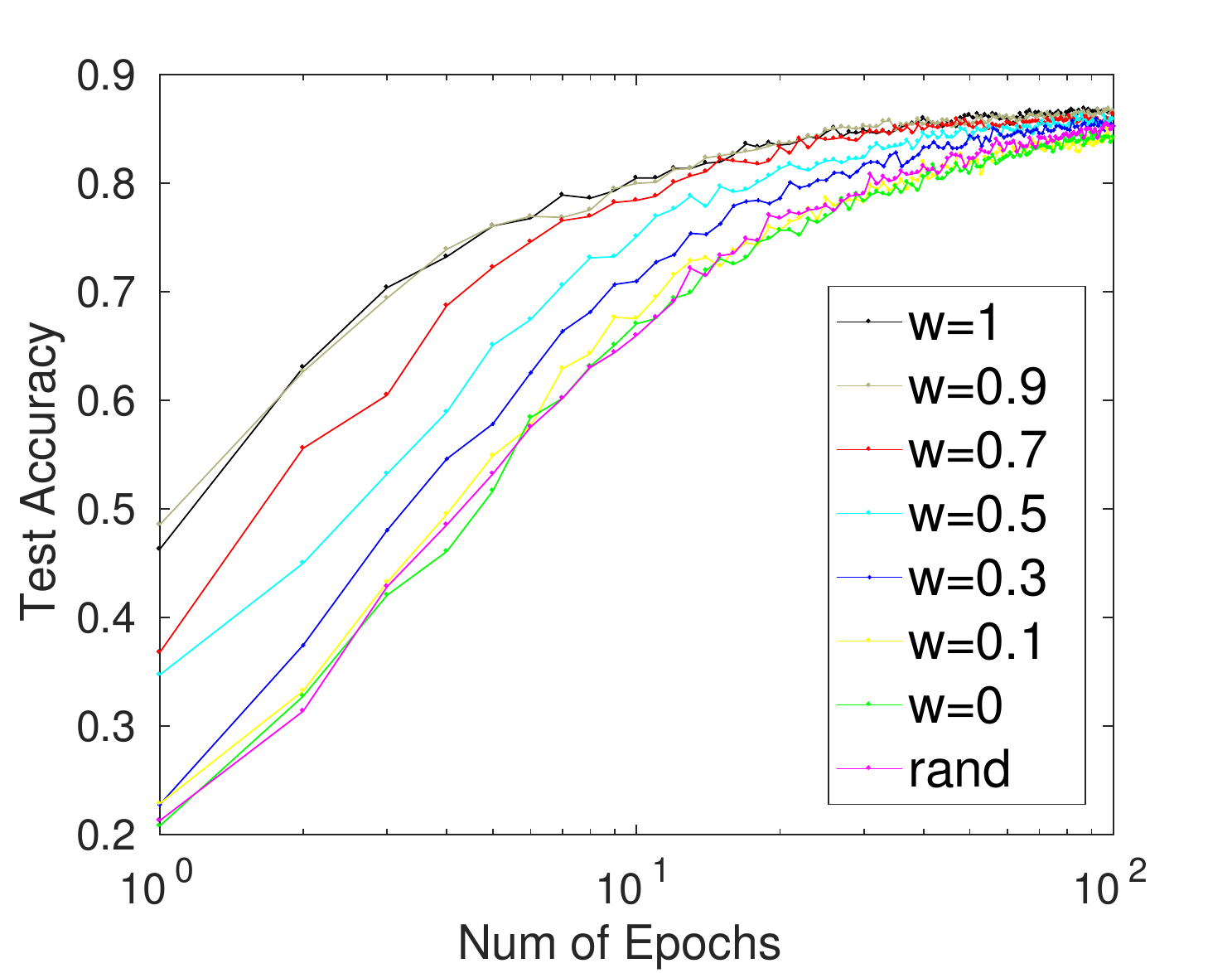}
}
\vspace{-8pt}
\subfigure[ k=102, Top3: 1, 0.9, 0.7; \protect\newline
Best: $86.5\%$ Baseline:$84.5\%$]{
\includegraphics[width=4.cm, height = 2.85cm]{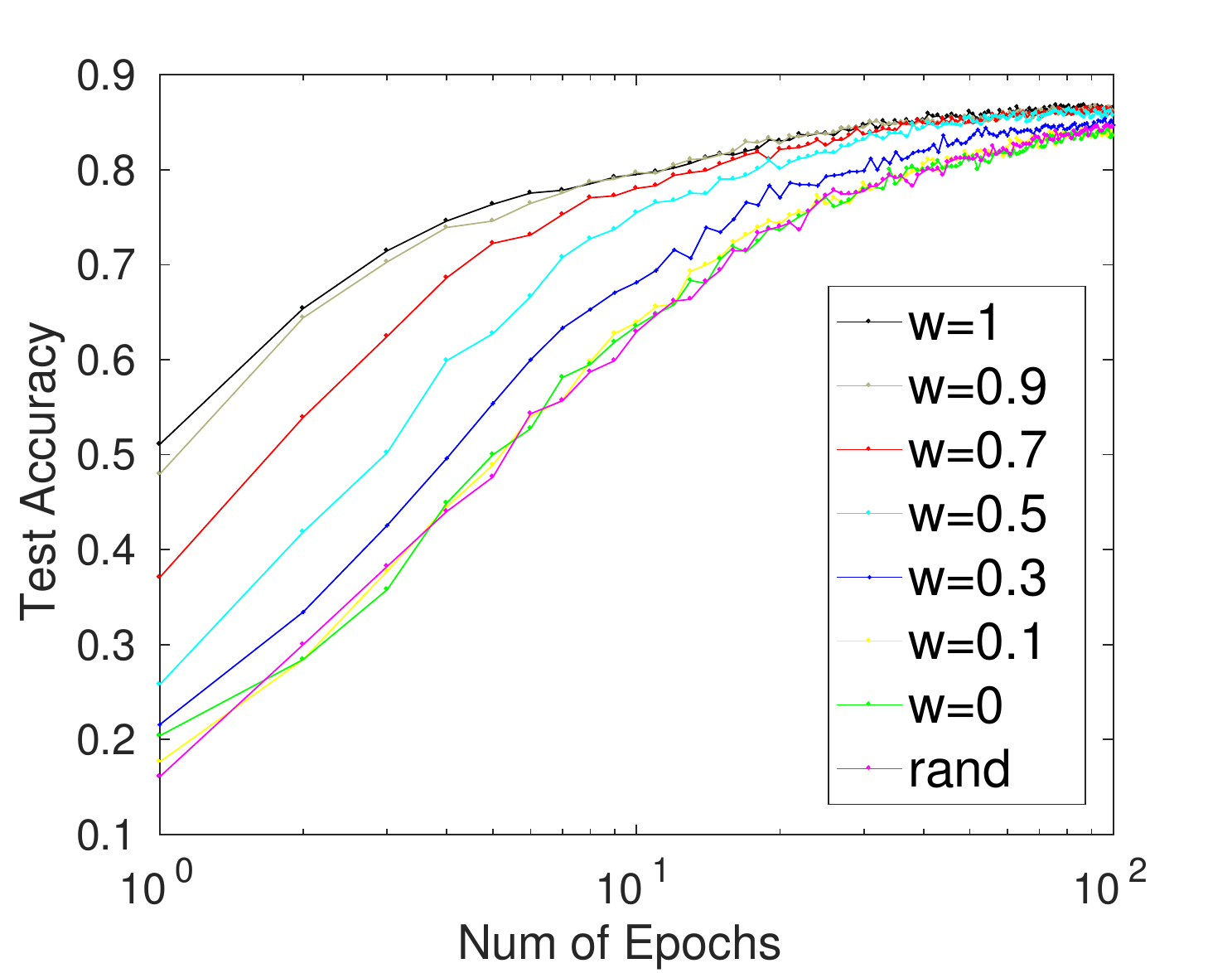}
}
\hspace{-10pt}
\subfigure[k=150, Top3: 0.7, 0.5, 0.9; \protect\newline
Best: $85.5\%$  Baseline:$83.1\%$]{
\includegraphics[width=4.cm, height = 2.85cm]{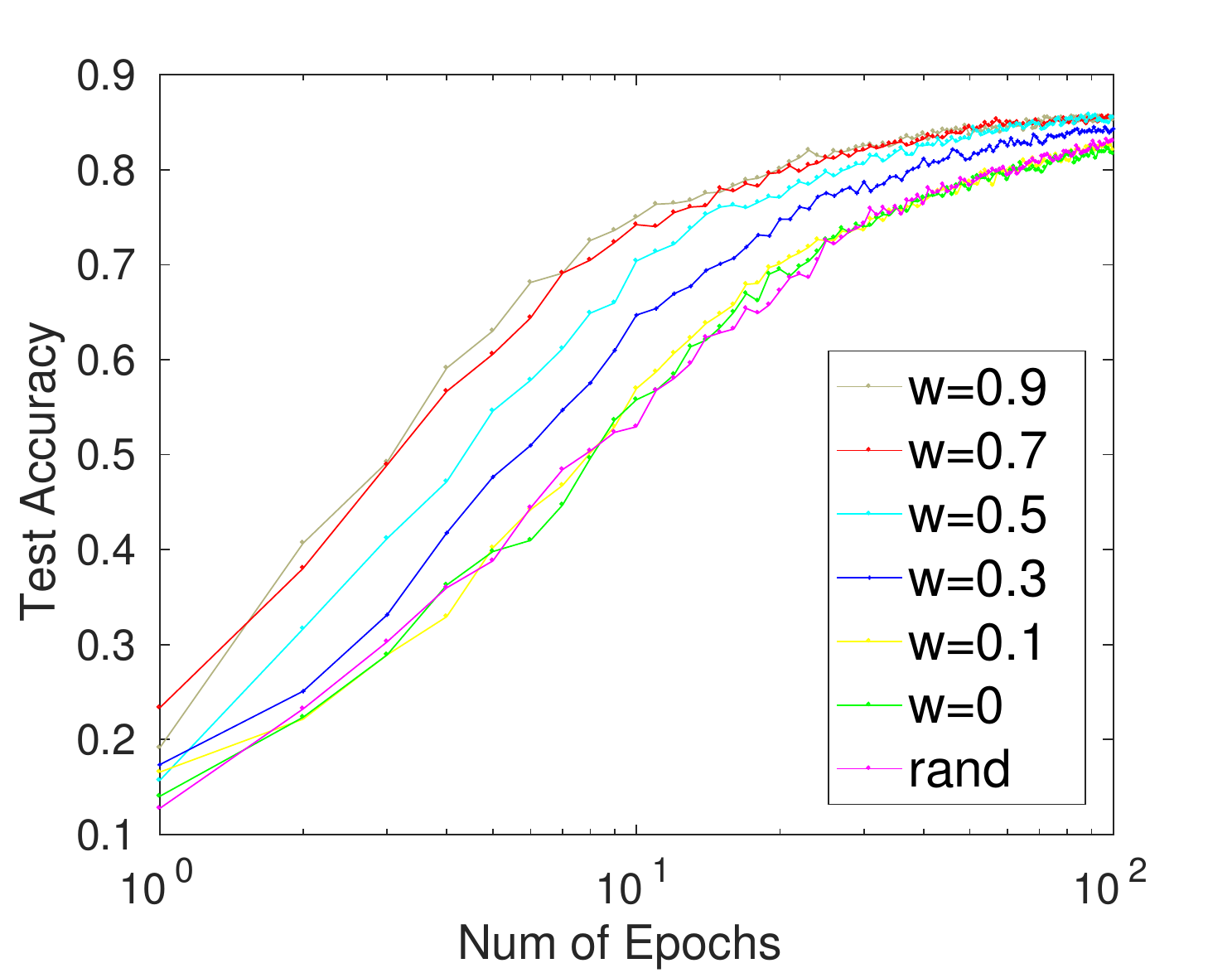}
}
%\subfigure[S 102]{
%\includegraphics[width=4.cm]{figs/Oxford_Flowers/K102_log.pdf}
%}
\caption{Test accuracy as a function of training epochs on the Oxford 102 multi-class classification task. We show DM-SGD for different values of $w$, with $w=1$ being biased stratified sampling (see Eq.~\ref{eq:definition_w} and the discussion below). The plot caption indicates the batch size $k$ and the three best performing values of $w$. 'Rand' indicates regular SGD sampling. We listed the final test accuracy after convergence, where "Best" indicates the best performance within our DM-SGD experiments, and "Baseline" indicates regular SGD as our baseline. The improvement is up to $5\%$.} 
\vspace{-8pt}
\label{fig:FlowerTestAcc}
\end{center}
\end{figure}
\vspace{-8pt}
Off-the-shelf CNN features \cite{sharif2014cnn} are used in this experiment. A pre-trained VGG16 network \cite{simonyan2014very} is used for the feature extraction. We use the first fully connected layer as features, since \cite{azizpour2015generic} shows that this layer is most robust. % among different tasks. 

The similarity kernel $L$ of the $k$-DPP was constructed as follows. We chose a linear kernel $L=FF^\top$, where $F$ is a weighted concatenation of the fc1 features $X_{fc1}$ and the labels a one-hot-vector representation of the class label $H$,
\begin{equation}
F = [(1-w)X_{fc1} ~~ wH], \quad 0 \leq w \leq 1.
\label{eq:definition_w}
\end{equation}
This kernel construction enables the population to be balanced both among classes and within classes. When $w$ is large, the algorithm focuses more on the class labels. When $w$ is small, balancing is performed mostly based on the features. The weighting factor $w$ is a free parameter. As $w=1$ results in stratified sampling (see Theorem~\ref{thrm:stratified}), this baseline is naturally captured in our approach.

%For diversified mini-batch sampling, we construct the similarity kernel matrix $L$ with a linear kernel $L=FF^\top$. When the class label information is available, we would like to take both the class label information and the data feature information into consideration. This way enables the population to be balanced both among classes and within classes. Hence, we construct $F$ by concatenating the fc1 feature and one-hot-vector representation of class label $H$  together with different weights $w$ as $F=[(1-w)X_{fc1} ~~ wH]$, where $0 \le w \le 1$. When $w$ is large, the balancing focuses more on the class label. When $w$ is small, the balancing is performed in an unsupervised manner.
%as in the previous experiment.

In this setting, the class label is a natural criterion to divide the data into strata. One can then re-sample the same amount of data from each stratum in order to re-balance the data set. Such a mechanism constrains the mini-batch size to be $k=sM$ where $M$ is the number of classes/strata and $s$ is a positive integer.
As proved in Section \ref{sec:theory}, when $k=M$ and $w=1$, DM-SGD is equivalent to this type of (biased) stratified sampling. 
%With $k<M$, DM-SGD is a direct generalization for stratified sampling with data being divided to strata by class labels. 

Figure \ref{fig:hist_flower} shows the percentage of data in each class for the original dataset and with  the balanced dataset. 
It shows that with larger $w$, the dataset is more balanced among classes. More examples are shown in the supplementary material.  

%\paragraph{Softmax.}~ 
We demonstrate this application with a standard linear Softmax classifier for multi-class classification. 
%A standard Softmax classifier is implemented using Tensorflow with cross-entropy loss. 
In our case, the inputs are the off-the-shelf CNN fc1 feature ($X_{fc1}$). We can also view this procedure as fine-tuning a neural network.

Figure \ref{fig:FlowerTestAcc} shows how the test accuracy changes with respect to each training epoch. We compare the DM-SGD with different weights against random sampling. The learning rate schedule is kept the same among different experiments. Different mini-batch sizes $k$ are used, which is shown in the caption of each panel in the figure. We can see that with DM-SGD, we can reach a high model performance more rapidly. Additionally, for a classification task, balancing data with respect to classes is important since the performance is better in general for bigger $w$. On the other hand, the feature information is essential as well since the best performance is mostly obtained with $w=0.9$ and $w=0.7$. 
Comparing these plots, we can see that the performance benefits more when the mini-batch size is comparably small. Small mini-batches in general are preferred due to low cost and our method can maximize the usage of small mini-batches.

%\begin{figure}
%\begin{center}
%\includegraphics[width=8.5cm]{figs/K102_diffW_testAcc_new.pdf}
%\label{K=102}
%\end{center}
%\end{figure}

%\begin{figure}
%\begin{center}
%\includegraphics[width=8.5cm]{figs/K150_diffW_testAcc_Partial.pdf}
%\label{K=150 partial}
%\end{center}
%\end{figure}

% \begin{figure}
% \begin{center}
% \subfigure[Orginal distribution]{
% \includegraphics[width=6.01cm]{figs/GroundTruth_hist.jpg}
% }
% \subfigure[resampled population]{
% \includegraphics[width=6.01cm]{figs/hotlabel_BP_hist.jpg}
% }
% \end{center}
% \end{figure}
% \begin{figure}
% \begin{center}
% \subfigure[Orginal distribution]{
% \includegraphics[width=6.01cm]{figs/test_Acc_Epoch.jpg}
% }
% \subfigure[resampled population]{
% \includegraphics[width=6.01cm]{figs/test_Acc_Step.jpg}
% }
% \subfigure[resampled population]{
% \includegraphics[width=6.01cm]{figs/test_Acc_Step_logfig.jpg}
% }
% \end{center}
% \end{figure}

%%%%%%%%%%%%%%%%%%%%%%%%%%%%%%%%%%%%%%%%%%%%%%%%%%%%%%%%%%%%%%%%%%%%%%%%%%%%%%%%%%%%%%%%%%%%%%%%%%%%%%%%%%%%%%%%%%%%%%%%%%%%%%%%%%%%%%%%%%%%%%%%%%%%%%
\subsection{CNN CLASSIFICATION ON MNIST}
\label{sec:cnn}
Finally, we show the performance of our method in a scenario where the dataset is balanced, which is less preferable scenario for DM-SGD. Here we consider the MNIST dataset \cite{lecun1998gradient}, which contains approximately the same number of examples per hand-written digits. 
\begin{figure}[t]
\begin{center}
\subfigure[k=10]{
\includegraphics[width=3.91cm, height = 2.8cm]{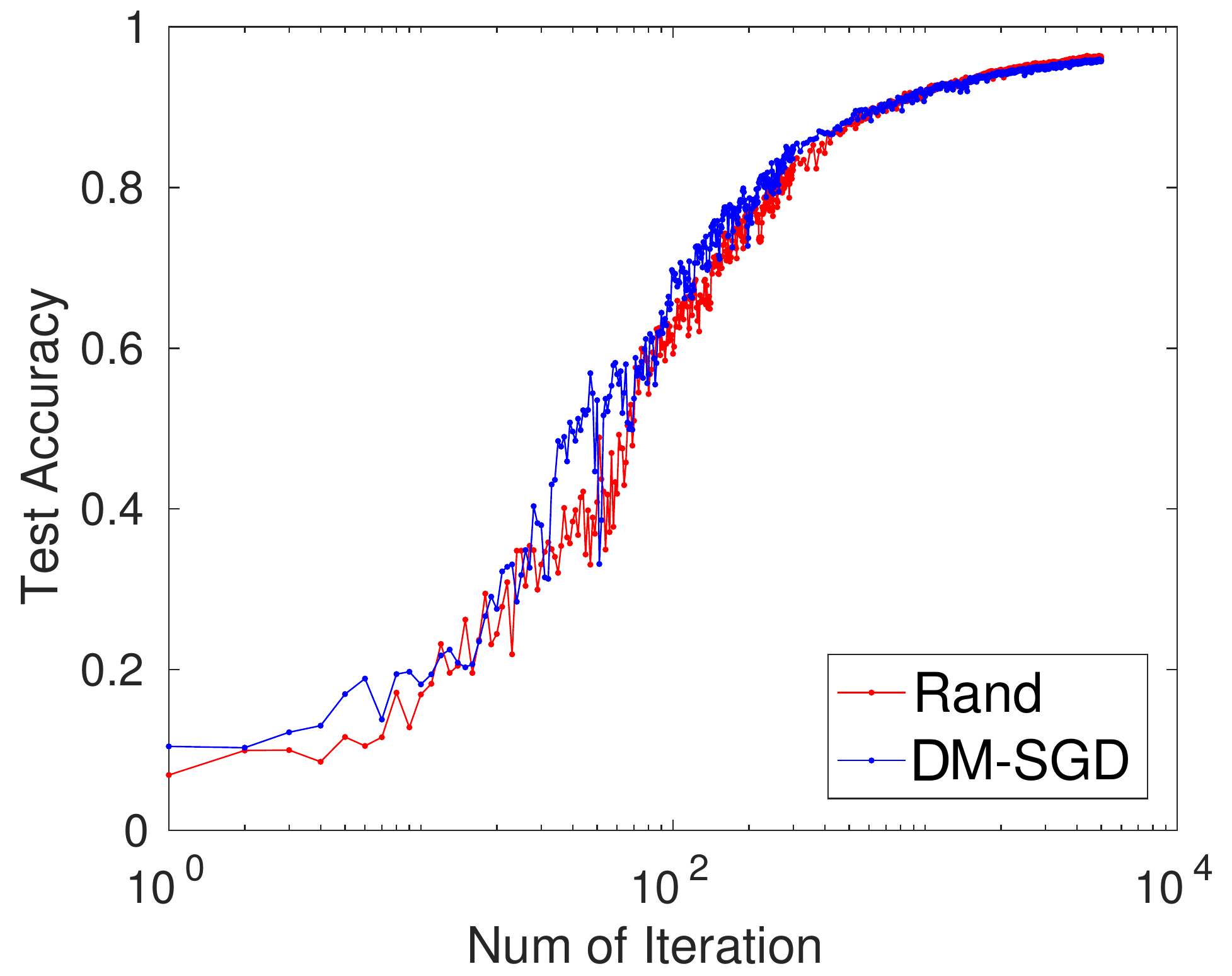}
}
\hspace{-5pt}
\subfigure[k=200]{
\includegraphics[width=3.91cm, height = 2.8cm]{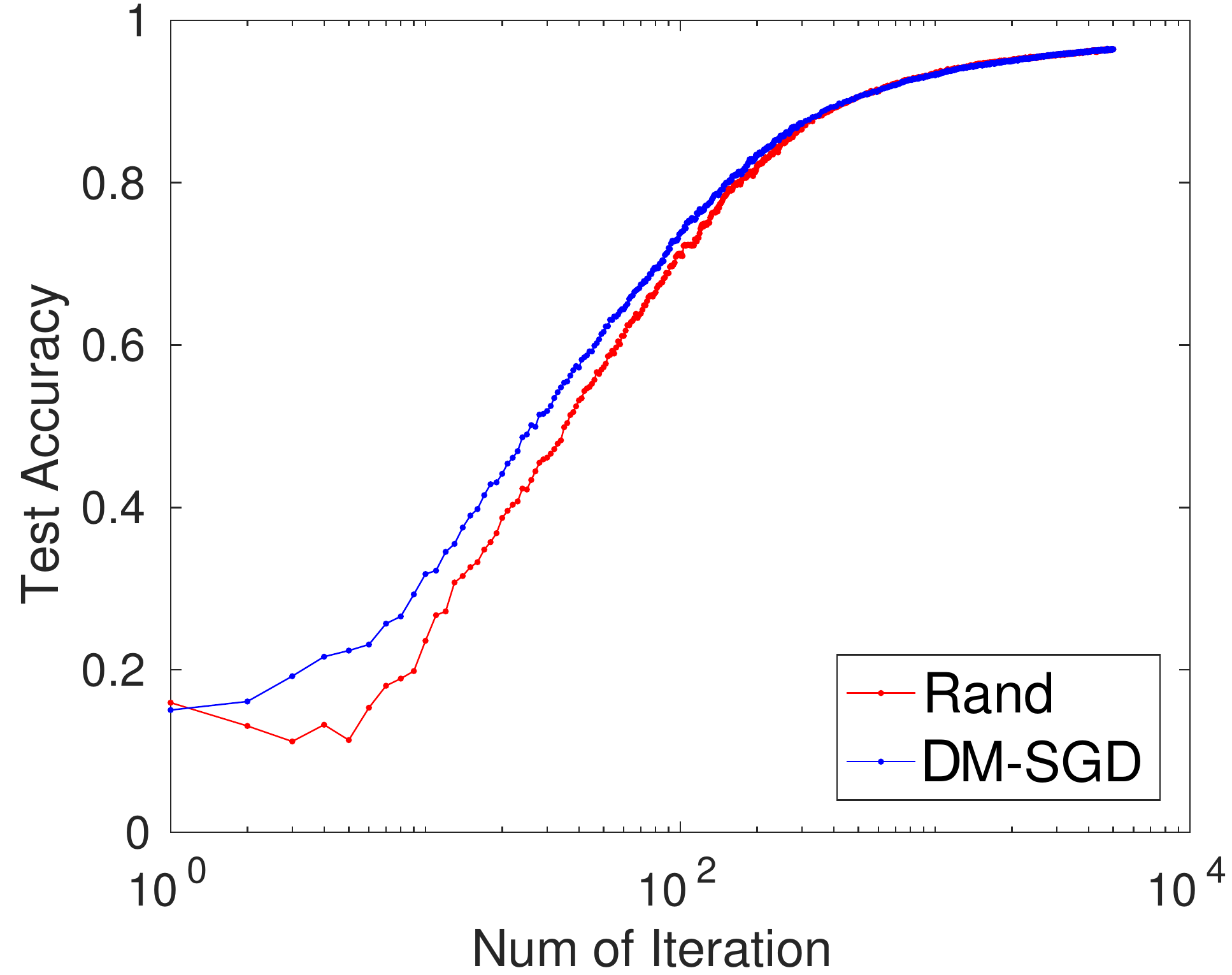}
}
\caption{Same quantities shown as in Fig.~\ref{fig:FlowerTestAcc}, but for the MNIST data set, which is more balanced.}
\label{fig:MNIST}
\end{center}
\vspace{-15pt}
\end{figure}
Since our method is independent of the model, we can use any low level data statistics. Here, we demonstrate DM-SGD with raw data features and apply it to training a CNN. 
Here, we construct the similarity kernel using a RBF kernel. For the low level feature, we use the normalized raw pixel value $X$ directly. To encode both class information and label information, we use  $F=[(1-w)X ~~ wH]$ to compute the similarities matrix, where $w=0.5$ for this experiment. 
We use half of the training data from MNIST to train a 5-layer CNN as in \cite{tensorflowCode}. Figure \ref{fig:MNIST} shows the test accuracy from each iteration with mini-batch size $10$ and $200$ respectively. We can see that even if the data are balanced, DM-SGD still performs better than random sampling due to its variance reduction property.  

%{\color{red} TODO w=0.9, S small. Believe that the performance will be better.}
\section{CONCLUSION}
We proposed a diversified mini-batch sampling scheme based on determinantal point processes. Our method, DM-SGD, builds on a similarity matrix between the data points and suppresses the co-occurance of similar data points in the same mini-batch. This leads to a training outcome which generalizes better to unseen data. We also derived sufficient conditions under which the method reduces the variance of the stochastic gradient, leading to faster learning. We showed that our approach generalizes both stratified sampling and pre-clustering. 
In the future, 
we will explore the possibility to further improve the efficiency of the algorithm with data reweighing
\cite{mandt2016variationalTem} and tackle imbalanced learning problems involving different modalities for supervised \cite{zhang14how} and multi-modal  \cite{joulin2016learning} settings. 

{
\footnotesize
\bibliography{ref}

\begin{thebibliography}{10}

\bibitem{facebookcode}
Image classification with imagenet.
\newblock
  \url{https://github.com/soumith/imagenet-multiGPU.torch/blob/master/dataset.lua}.

\bibitem{tensorflowCode}
Multilayer convolutional network.
\newblock \url{https://www.tensorflow.org/get_started/mnist/pros}.

\bibitem{R8dataset}
{R8 dataset}.
\newblock \url{http://csmining.org/index.php/r52-and-r8-of-reuters-21578.html}.

\bibitem{affandi2013nystrom}
R.~H. Affandi, A.~Kulesza, E.~B. Fox, and B.~Taskar.
\newblock Nystrom approximation for large-scale determinantal processes.
\newblock In {\em AISTATS}, pages 85--98, 2013.

\bibitem{azizpour2015generic}
H.~Azizpour, A.~Sharif~Razavian, J.~Sullivan, A.~Maki, and S.~Carlsson.
\newblock From generic to specific deep representations for visual recognition.
\newblock In {\em CVPR WS}, pages 36--45, 2015.

\bibitem{blei2003latent}
D.~M. Blei, A.~Y. Ng, and M.~I. Jordan.
\newblock {Latent Dirichlet Allocation}.
\newblock {\em JMLR}, 3:993--1022, 2003.

\bibitem{bottou2010large}
L.~Bottou.
\newblock Large-scale machine learning with stochastic gradient descent.
\newblock In {\em COMPSTAT}, pages 177--186. 2010.

\bibitem{csiba2016importance}
D.~Csiba and P.~Richtarik.
\newblock Importance sampling for minibatches.
\newblock {\em arXiv:1602.02283}, 2016.

\bibitem{duchi2011adaptive}
J.~Duchi, E.~Hazan, and Y.~Singer.
\newblock Adaptive subgradient methods for online learning and stochastic
  optimization.
\newblock {\em JMLR}, 12(Jul):2121--2159, 2011.

\bibitem{freund1995desicion}
Y.~Freund and R.~E. Schapire.
\newblock A desicion-theoretic generalization of on-line learning and an
  application to boosting.
\newblock In {\em EuroCOLT}, pages 23--37. Springer, 1995.

\bibitem{fu2017CPSGMCMC}
T.F Fu and Z.H. Zhang.
\newblock {CPSG-MCMC}: Clustering-based preprocessing method for stochastic
  gradient {MCMC}.
\newblock In {\em AISTATS}, 2017.

\bibitem{he2009learning}
H.~He and E.~A. Garcia.
\newblock Learning from imbalanced data.
\newblock {\em TKDE}, 21(9):1263--1284, 2009.

\bibitem{hoffman2010online}
M.~Hoffman, F.~R. Bach, and D.~M. Blei.
\newblock Online learning for latent dirichlet allocation.
\newblock In {\em NIPS}, pages 856--864, 2010.

\bibitem{jordan1999introduction}
M.~I. Jordan, Z.~Ghahramani, T.~S. Jaakkola, and L.~K. Saul.
\newblock An introduction to variational methods for graphical models.
\newblock {\em Machine learning}, 37(2):183--233, 1999.

\bibitem{joulin2016learning}
A.~Joulin, L.~van~der Maaten, A.~Jabri, and N.~Vasilache.
\newblock Learning visual features from large weakly supervised data.
\newblock In {\em ECCV}, 2016.

\bibitem{kingma2014adam}
D.~Kingma and J.~Ba.
\newblock Adam: A method for stochastic optimization.
\newblock {\em arXiv:1412.6980}, 2014.

\bibitem{kingman1993poisson}
J.~F.~C. Kingman.
\newblock {\em Poisson processes}.
\newblock Wiley Online Library, 1993.

\bibitem{kulesza2011k}
A.~Kulesza and B.~Taskar.
\newblock {k-DPPs: Fixed-size determinantal point processes}.
\newblock In {\em ICML}, pages 1193--1200, 2011.

\bibitem{kulesza2012determinantal}
A.~Kulesza and B.~Taskar.
\newblock Determinantal point processes for machine learning.
\newblock {\em arXiv:1207.6083}, 2012.

\bibitem{kwok2012priors}
J.~T. Kwok and R.~P. Adams.
\newblock Priors for diversity in generative latent variable models.
\newblock In {\em NIPS}, pages 2996--3004, 2012.

\bibitem{lecun1998gradient}
Y.~LeCun, L.~Bottou, Y.~Bengio, and P.~Haffner.
\newblock Gradient-based learning applied to document recognition.
\newblock {\em Proceedings of the IEEE}, 86(11):2278--2324, 1998.

\bibitem{lee2016individualness}
D.~Lee, G.~Cha, M.~H. Yang, and S.~Oh.
\newblock Individualness and determinantal point processes for pedestrian
  detection.
\newblock In {\em ECCV}, pages 330--346, 2016.

\bibitem{li2016fast}
C.~T. Li, S.~Jegelka, and S.~Sra.
\newblock Fast {DPP} sampling for nystr\"oom with application to kernel
  methods.
\newblock {\em arXiv:1603.06052}, 2016.

\bibitem{li2015efficient}
C.T. Li, S.~Jegelka, and S.~Sra.
\newblock Efficient sampling for k-determinantal point processes.
\newblock {\em arXiv:1509.01618}.

\bibitem{macchi1975coincidence}
O.~Macchi.
\newblock The coincidence approach to stochastic point processes.
\newblock {\em Advances in Applied Probability}.

\bibitem{mandt2014smoothed}
S.~Mandt and D.~M. Blei.
\newblock Smoothed gradients for stochastic variational inference.
\newblock In {\em NIPS}, pages 2438--2446, 2014.

\bibitem{mandt2017stochastic}
S.~Mandt, M.~D. Hoffman, and D.~M. Blei.
\newblock Stochastic gradient descent as approximate {Bayesian} inference.
\newblock {\em arXiv:1704.04289}, 2017.

\bibitem{mandt2016variationalTem}
S.~Mandt, J.~McInerney, F.~Abrol, R.~Ranganath, and D.~M. Blei.
\newblock {Variational Tempering}.
\newblock In {\em AISTATS}, pages 704--712, 2016.

\bibitem{mckay1979comparison}
M.~D. McKay, R.~J. Beckman, and W.~J. Conover.
\newblock Comparison of three methods for selecting values of input variables
  in the analysis of output from a computer code.
\newblock {\em Technometrics}, 21(2):239--245, 1979.

\bibitem{neyman1934two}
J.~Neyman.
\newblock On the two different aspects of the representative method: the method
  of stratified sampling and the method of purposive selection.
\newblock {\em Journal of the Royal Statistical Society}, 97(4):558--625, 1934.

\bibitem{Nilsback08}
M.~E. Nilsback and A.~Zisserman.
\newblock Automated flower classification over a large number of classes.
\newblock In {\em ICVGIP}, 2008.

\bibitem{paisley2012variational}
J.~Paisley, D.~Blei, and M.~Jordan.
\newblock Variational bayesian inference with stochastic search.
\newblock {\em arXiv:1206.6430}, 2012.

\bibitem{perekrestenko2017faster}
D.~Perekrestenko, V.~Cevher, and M.~Jaggi.
\newblock Faster coordinate descent via adaptive importance sampling.
\newblock In {\em AISTATS}, 2017.

\bibitem{polyak1964some}
B.~T. Polyak.
\newblock Some methods of speeding up the convergence of iteration methods.
\newblock {\em USSR Computational Mathematics and Mathematical Physics},
  4(5):1--17, 1964.

\bibitem{ranganath2014black}
R.~Ranganath, S.~Gerrish, and D.~M. Blei.
\newblock Black box variational inference.
\newblock In {\em AISTATS}, pages 814--822, 2014.

\bibitem{robbins1985convergence}
H.~Robbins and D.~Siegmund.
\newblock A convergence theorem for non negative almost supermartingales and
  some applications.
\newblock In {\em Herbert Robbins Selected Papers}, pages 111--135. Springer,
  1985.

\bibitem{robertson2004understanding}
S.~Robertson.
\newblock Understanding inverse document frequency: on theoretical arguments
  for {IDF}.
\newblock {\em Journal of documentation}, 60(5):503--520, 2004.

\bibitem{salimans2014using}
T.~Salimans and D.A. Knowles.
\newblock On using control variates with stochastic approximation for
  variational bayes and its connection to stochastic linear regression.
\newblock {\em arXiv:1401.1022}.

\bibitem{schmidt2015non}
M.~Schmidt, R.~Babanezhad, M.~O. Ahmed, A.~Defazio, A.~Clifton, and A.~Sarkar.
\newblock Non-uniform stochastic average gradient method for training
  conditional random fields.
\newblock In {\em AISTATS}, 2015.

\bibitem{schmidt2013minimizing}
M.~Schmidt, N.~Le~Roux, and F.~Bach.
\newblock Minimizing finite sums with the stochastic average gradient.
\newblock {\em Mathematical Programming}, pages 1--30, 2013.

\bibitem{sharif2014cnn}
A.~Sharif~Razavian, H.~Azizpour, J.~Sullivan, and S.~Carlsson.
\newblock Cnn features off-the-shelf: an astounding baseline for recognition.
\newblock In {\em CVPR WS}, pages 806--813, 2014.

\bibitem{simonyan2014very}
K.~Simonyan and A.~Zisserman.
\newblock Very deep convolutional networks for large-scale image recognition.
\newblock {\em arXiv:1409.1556}, 2014.

\bibitem{tieleman2012lecture}
T.~Tieleman and G.~Hinton.
\newblock {Lecture 6.5-RMSPROP}: Divide the gradient by a running average of
  its recent magnitude.
\newblock {\em COURSERA: Neural networks for machine learning}.

\bibitem{wang2013variance}
C.~Wang, X.~Chen, A.~J. Smola, and E.~P. Xing.
\newblock Variance reduction for stochastic gradient optimization.
\newblock In {\em NIPS}, pages 181--189, 2013.

\bibitem{xie2015diversifying}
P.T. Xie, Y.T. Deng, and E.~Xing.
\newblock Diversifying restricted boltzmann machine for document modeling.
\newblock In {\em ACM SIGKDD}, pages 1315--1324. ACM, 2015.

\bibitem{zeiler2012adadelta}
M.~D. Zeiler.
\newblock {ADADELTA}: an adaptive learning rate method.
\newblock {\em arXiv:1212.5701}, 2012.

\bibitem{zhang14how}
C.~Zhang and H.~Kjellstr{\"o}m.
\newblock {How to Supervise Topic Models}.
\newblock In {\em ECCV WS}, 2014.

\bibitem{zhao2014accelerating}
P.L. Zhao and T.~Zhang.
\newblock Accelerating minibatch stochastic gradient descent using stratified
  sampling.
\newblock {\em arXiv:1405.3080}.

\bibitem{zhao2015stochastic}
P.L. Zhao and T.~Zhang.
\newblock Stochastic optimization with importance sampling for regularized loss
  minimization.
\newblock In {\em ICML}, pages 1--9, 2015.

\end{thebibliography}
\bibliographystyle{plain}
}
\clearpage
\appendix
\section{SUPPLEMENT}

Algorithm \ref{alg:EigVsample} shows the details of how to sample a mini-batch using $k$-DPP \cite{kulesza2012determinantal} which is used for the DM-SGD and DM-SVI algorithm in the paper.

\vspace{-7pt}
\begin{algorithm}[h]
\textbf{Input:} Mini-batch size $k$, eigendecomposition $\{ (v_n, \lambda_n)\}^N _{n=1} $ of similarity matrix $L$.\\
\textbf{Compute the elementary  symmetric polynomials}\\
$e_0^n \leftarrow 1 \forall n \in \{ 0,1,2, ..., N\} $\\
$e_0^l \leftarrow 1 \forall l\in \{ 1,2, ..., k\} $\\
\For{$l=1,2,...,k$}{
 \For{$n=1,2,..., N$}{
     $e_l^n \leftarrow e_l^{n-1} + \lambda_n e_{l-1}^{n-1}$
 }
}
\For{t=1 to Number of subset samples to generate}{
\textbf{Sampling $k$ eigenvectors $V$ with indices $J$}\\
$J \leftarrow \emptyset$\\
$l \leftarrow k$\\
\For{ $n= N,...,2,1$}{
\If{$l=0$}{
break;
}

\If{ $u \sim U[0,1]  \le \lambda_n \frac{ e_{l-1}^{n-1}} { e_l^n}$ }{
	$J \leftarrow J \cup \{n\}$ \\
	$l \leftarrow l -1$\\
}
}

\textbf{Sample  $k$ data points indexed by $Y$ using $V$.}\\
$ V \leftarrow \{ v_i\} _{i \in J}$\\
$Y \leftarrow \emptyset$\\
   \While{ $|V| > 0$}{
   	Select $i$ with $Pr(i) = \frac{1}{|V|} \sum_{v \in V} (v^{T} e_i)^2$\\
	$Y \leftarrow Y \cup i$
	$V \leftarrow V_{\perp}$, an orthonormal basis for the subspace of $V$ orthogonal to $e_i$\\
   }
\textbf{Output:} $Y$\\
}
\caption{Mini-batch Sampling }
 \label{alg:EigVsample}
\end{algorithm}
\vspace{-10pt}

Table \ref{tab:topwords_SVI}  and \ref{tab:topwords_BP_SVI} show the top words using $K=30$ for LDA using traditional SVI and our proposed  DM-SVI respectively. We can see that the topics that are learned by DM-SVI are more diverse and rare topics such as grain (colored in blue) are captured.

Figure \ref{fig:alldata} shows the synthetic data that are used in the LDA experiment. Each row represents a document and each column represents a word.  

\begin{figure}[h]
\centering
\includegraphics[width=8cm]{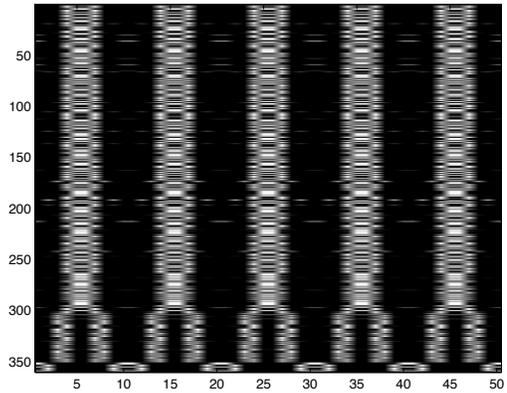}
\caption{Synthetic data used to train the LDA model in the main paper. Each row presents a document and each column represents a word. Documents share topics with highly imbalanced proportions.}
\label{fig:alldata}
\end{figure}
\begin{table}[p]
\small
\centering
		\begin{tabular}{| l | p{6cm} |}
		\hline
Topic 1 &	pct shares stake and group investment securities stock commission firm \\
Topic 2 &	year pct and for last lower growth debt profits company \\
Topic 3 &	and merger for will approval companies corp acquire into letter\\ 
Topic 4 &	and for canadian company management pacific bid southern court units \\
Topic 5 &	baker official and that treasury western policy administration study budget \\
Topic 6 &	and president for executive chief shares plc company chairman cyclops \\
Topic 7 &	bank pct banks rate rates money interest and reuter today \\
Topic 8 &	and unit inc sale sell reuter company systems corp terms \\
Topic 9 &	mln stg and reuter months year for plc market pretax \\
Topic 10 &	and national loan federal savings reuter association insurance estate real \\
Topic 11 &	trade and for bill not united imports that surplus south \\
Topic 12 &	and february for china january gulf issue month that last \\
Topic 13 &	market dollar that had and will exchange system currency west \\
Topic 14 &	dlrs quarter share for company earnings year per and fiscal \\
Topic 15 &	billion mln tax year profit credit marks francs net pct \\
Topic 16 &	usair inc twa reuter trust air department chemical diluted piedmont \\
Topic 17 &	and will union spokesman not two that reuter security port \\
Topic 18 &	offer share tender shares that general and gencorp dlrs not \\
Topic 19 &	and company for that board proposal group made directors proposed \\
Topic 20 &	that japan japanese and world industry government for told officials \\
Topic 21 &	american analysts and that analyst chrysler shearson express stock not \\
Topic 22 &	loss profit mln reuter cts net shr dlrs qtr year \\
Topic 23 &	mln dlrs and assets for dlr operations year charge reuter \\
Topic 24 &	mln net cts shr revs dlrs qtr year oper reuter \\
Topic 25 &	cts april reuter div pay prior record qtly march sets \\
Topic 26 &	dividend stock split for two reuter march payable record april \\
Topic 27 &	oil and prices crude for energy opec petroleum production bpd \\
Topic 28 &	agreement for development and years program technology reuter conditions agreed \\
Topic 29 &	and foreign that talks for international industrial exchange not since \\
Topic 30 &	corp inc acquisition will company common shares reuter stock purchase \\
		\hline
		\end{tabular}
		\caption{Top 10 words for each topics learned from LDA with traditional SVI.}
\label{tab:topwords_SVI}
\end{table}

\begin{table}[p]
\centering
\small
		\begin{tabular}{| l | p{6cm} |}
		\hline
		Topic 1 &	oil and that prices for petroleum dlrs energy crude field  \\
Topic 2 &	pct and that rate market banks term rates this will  \\
Topic 3 &	billion and pct mln group marks sales year capital rose \\ 
Topic 4 &	and saudi oil gulf that arabia december minister prices for \\
Topic 5 &	and dlrs debt for brazil southern mln will medical had \\
Topic 6 &	\color{blue} and grain that will futures for program farm certificates agriculture \\
Topic 7 &	bank banks rate and pct interest rates for foreign banking \\
Topic 8 &	and union for national seamen california port security that strike \\
Topic 9 &	and trade that for dollar deficit gatt not exports economic \\
Topic 10 &	and financial for sale inc services reuter systems agreement assets \\
Topic 11 &	dollar and for yen mark march that dealers sterling market \\
Topic 12 &	and for south unit equipment reuter two will state corp \\
Topic 13 &	and firm stock company will for pct not share that \\
Topic 14 &	and world that talks economic official for countries system monetary \\
Topic 15 &	and gencorp for offer general company partners that dlrs share \\
Topic 16 &	mln canada canadian stg and pct will air that royal \\
Topic 17 &	usair and twa that analysts not for pct analyst piedmont \\
Topic 18 &	and that for companies not years study this areas overseas \\
Topic 19 &	trade and bill for house that reagan foreign states committee \\
Topic 20 &	company dlrs offer stock and for corp share shares mln \\
Topic 21 &	dlrs year and quarter company for earnings will tax share\\ 
Topic 22 &	mln cts net loss dlrs profit reuter shr year qtr \\
Topic 23 &	exchange paris and rates that treasury baker allied for western \\
Topic 24 &	and shares inc for group dlrs pct offer reuter share \\
Topic 25 &	merger and that pacific texas hughes baker commerce for company \\
Topic 26 &	and american company subsidiary china french reuter pct for owned \\
Topic 27 &	japan japanese and that trade officials for government industry pact \\
Topic 28 &	oil opec mln bpd prices production ecuador and output crude \\
Topic 29 &	and that had shares block for mln government not san \\
Topic 30 &	mln pct and profits dlrs year for billion company will \\
		\hline
		\end{tabular}
		\caption{Top 10 words for each topics learned from LDA with DM-SVI.}
			\label{tab:topwords_BP_SVI}
\end{table}

The sampling time in seconds for the R8 dataset is listed in Table \ref{tab:R8time_sec}. There are 5485 training documents. The first row in the table shows the sampling time for different mini-batch sizes k and different versions of k-DPP sampling. In practice, we use the original implementation from \cite{li2016fast} with $M=100$. To compare with the traditional k-DPP, we listed the elapsed time with \cite{kulesza2012determinantal}. The last row shows the running time per local LDA update, excluding sampling.

% Table 3: The sampling time (in sec) for LDA on the R8 dataset with different mini-batch sizes

%     Size 			 	                   | k = 10 	 | k =30		| k=50		| k=80		| 
% Fast k-DPP CPU/elapsed time	   |0.001		| 0.0139		| 0.0541		| 0.2199		|
% k-DPP elapsed time   		           | 0.0098	|0.1468		| 0.6438		| 2.6698		|
% k-DPP CPU time        		           |0.0863	| 0.15228	| 6.8593		| 28.474		|
% LDA updates without sampling   |0.877695	| 1.25303	| 1.64137	|2.23116	|

\begin{table}
\begin{tabular}{ l  c c c c} 
 \toprule
 Size  & k = 10 & k =30 & k=50	& k=80  \\
 \cmidrule(lr){1-5}
Fast k-DPP  & 0.001	& 0.0139 &0.0541 & 0.2199 \\
k-DPP  & 0.0098 &0.1468 & 0.6438& 2.6698\\
LDA &0.8777& 1.2530& 1.6414	&2.2312 \\
 \bottomrule
\end{tabular}
\medskip
\caption{\footnotesize Sampling time (in sec) for LDA on the R8 dataset with different mini-batch sizes.}
\label{tab:R8time_sec}
\end{table}

The computational time for training a neural network highly depends on the network structure and implementation details. For example, when using only one softmax layer as in the flower experiment, the cost per gradient step is in the milliseconds. In this setup, k-DPP is not effective from a runtime perspective, but still results in better final classification accuracies. However, the cost for each gradient step for a simple 5 layer NN as in the MNIST experiment with $K=100$ is 1.294 seconds. In the latter case, this time is comparable to k-DPP sampling (0.7941 sec) see Table \ref{tab:MNIST_sec}. We thus expect our methods to benefit expensive models and imbalanced training datasets more.

\begin{table}
\begin{tabular}{ l  c c c} 
 \toprule
 Size  & k = 10 & k =100 & k=200  \\
 \cmidrule(lr){1-4}
Fast k-DPP  &  0.0012 & 0.7941 & 5.4216 \\
NN cost &0.166948 & 1.29452& 2.64811	 \\
 \bottomrule
\end{tabular}
\medskip
\caption{\footnotesize Five Layer NN trained on MNIST with different mini-batch sizes.
Top row: sampling time (in sec) using the fast k-DPP approach. Bottom row: run time for each update step (excluding mini-batch sampling).}
\label{tab:MNIST_sec}
\end{table}

\begin{figure}[h]
\begin{center}
\subfigure[Original]{
\includegraphics[width=3.2cm, height=2.3cm]{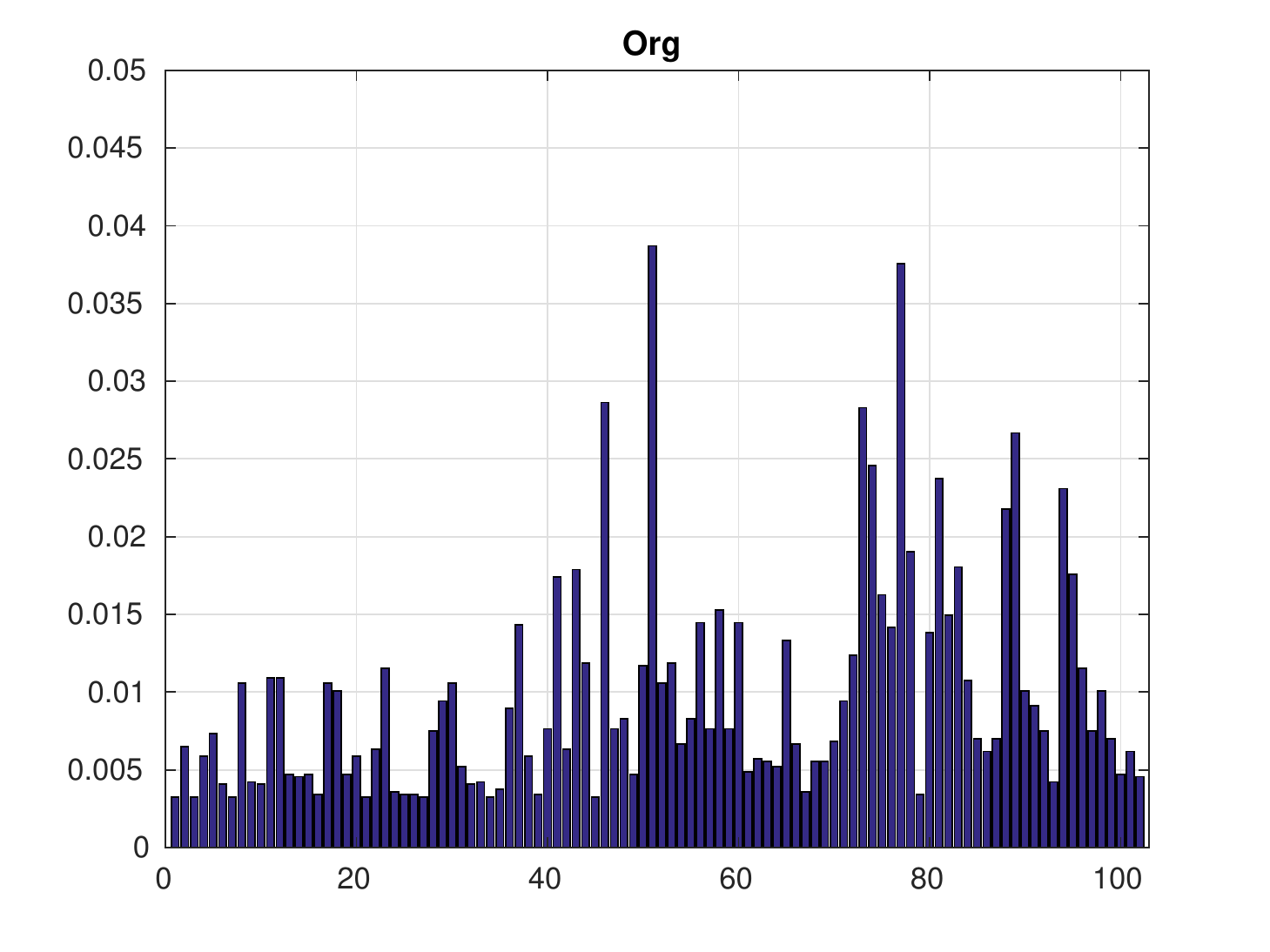}
}
\subfigure[w=0]{
\includegraphics[width=3.2cm, height=2.3cm]{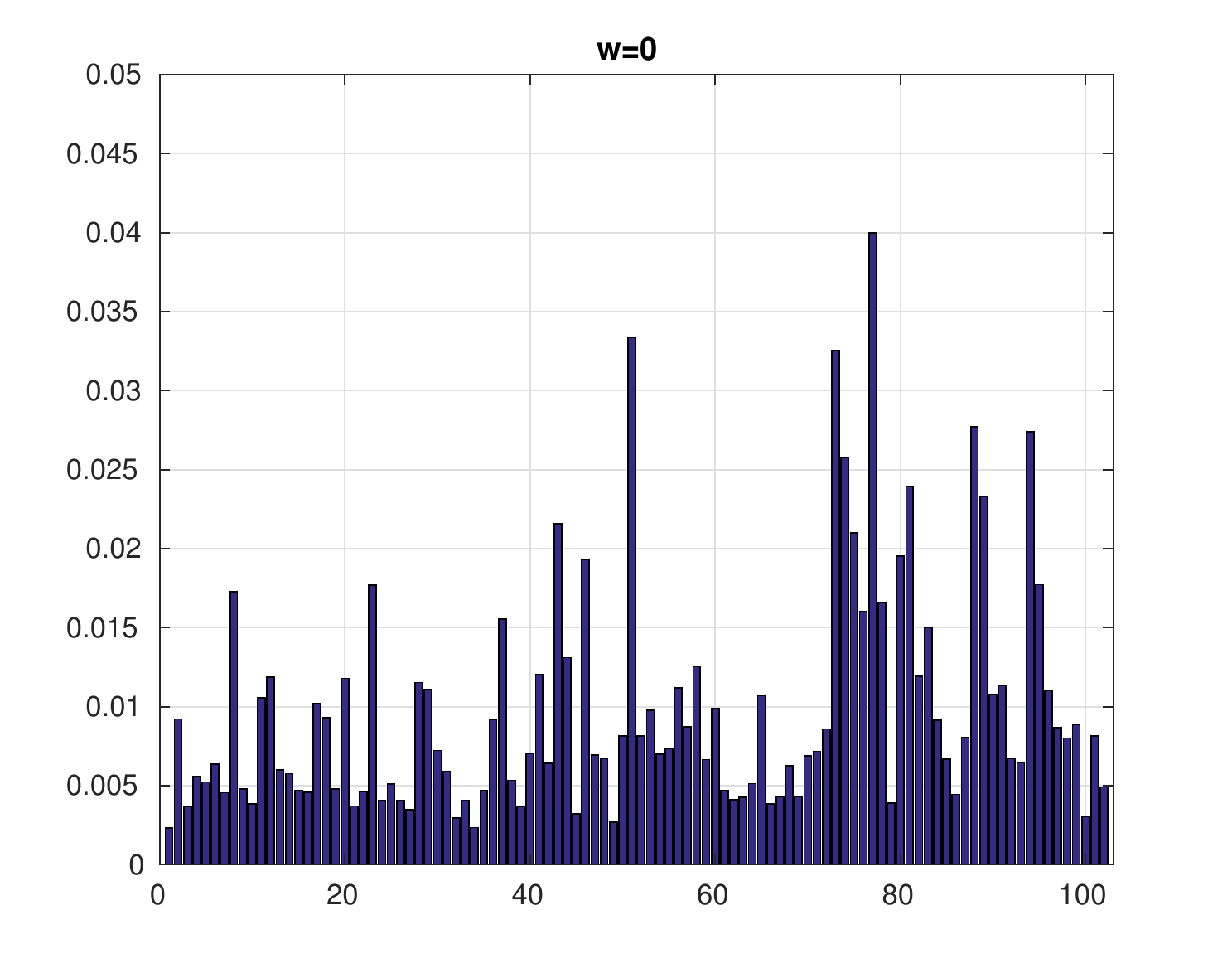}
}
\vspace{-10pt}

\subfigure[w=0.1]{
\includegraphics[width=3.2cm, height=2.3cm]{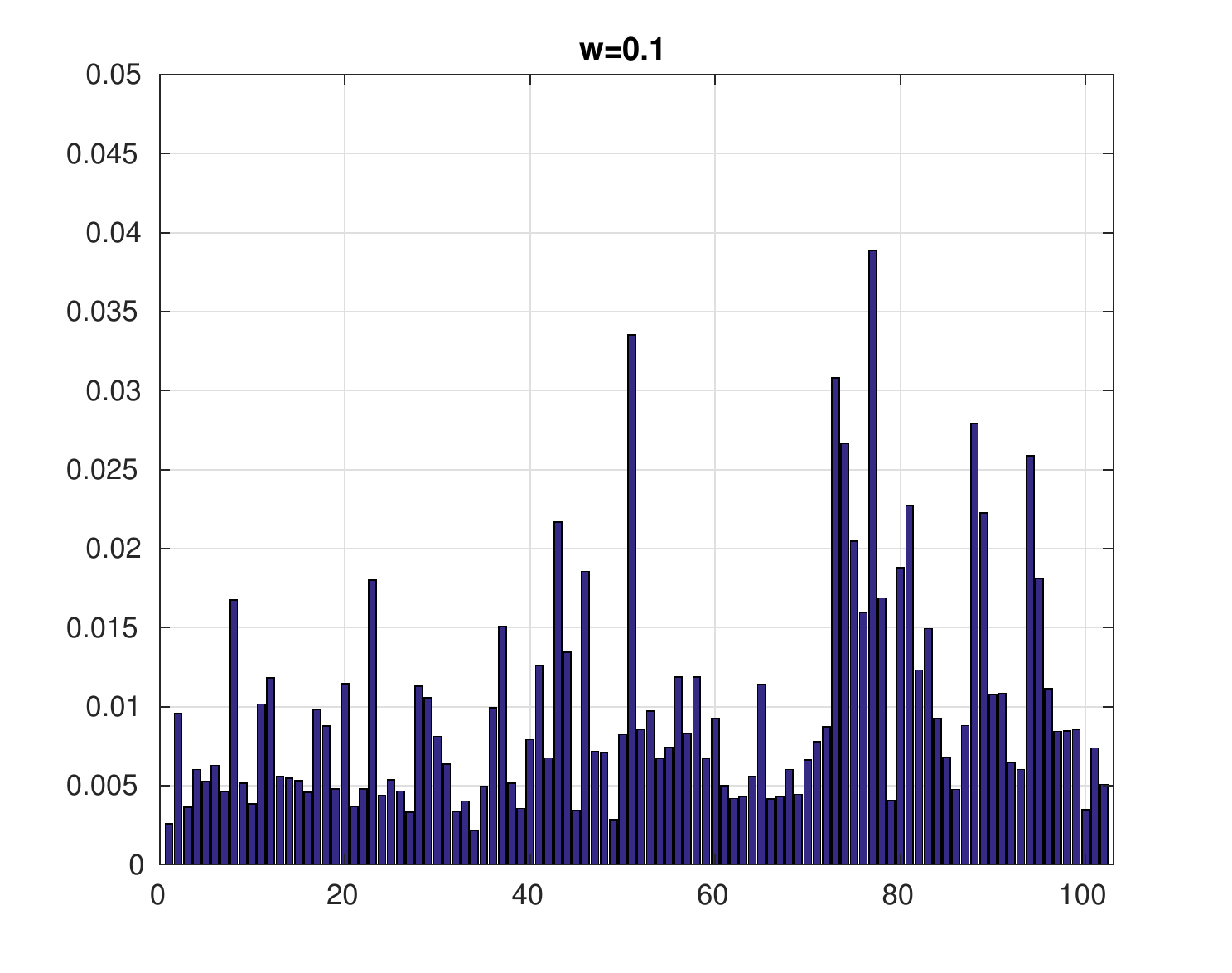}
}
\subfigure[w=0.3]{
\includegraphics[width=3.2cm, height=2.3cm]{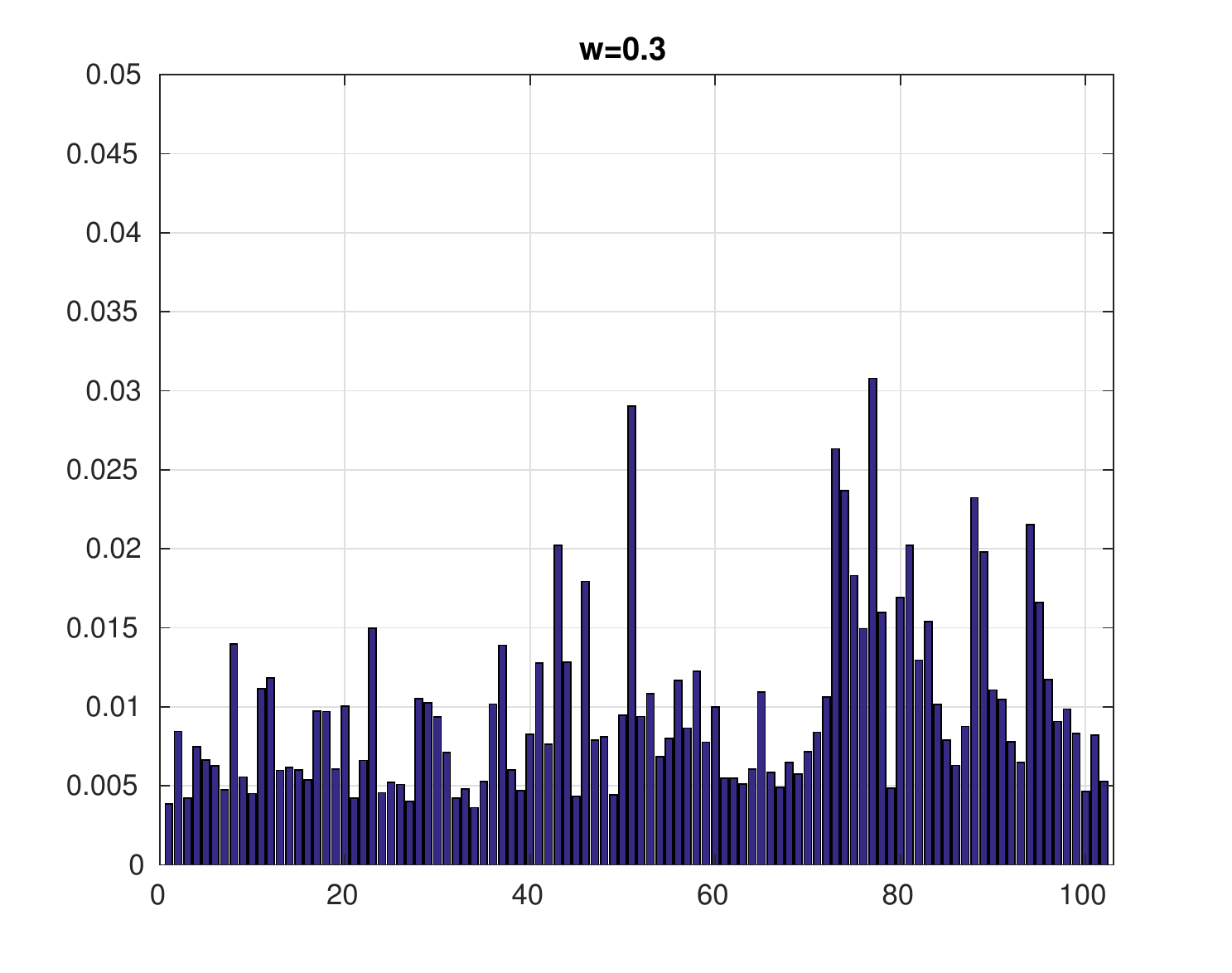}
}
\vspace{-10pt}

\subfigure[w=0.5]{
\includegraphics[width=3.2cm, height=2.3cm]{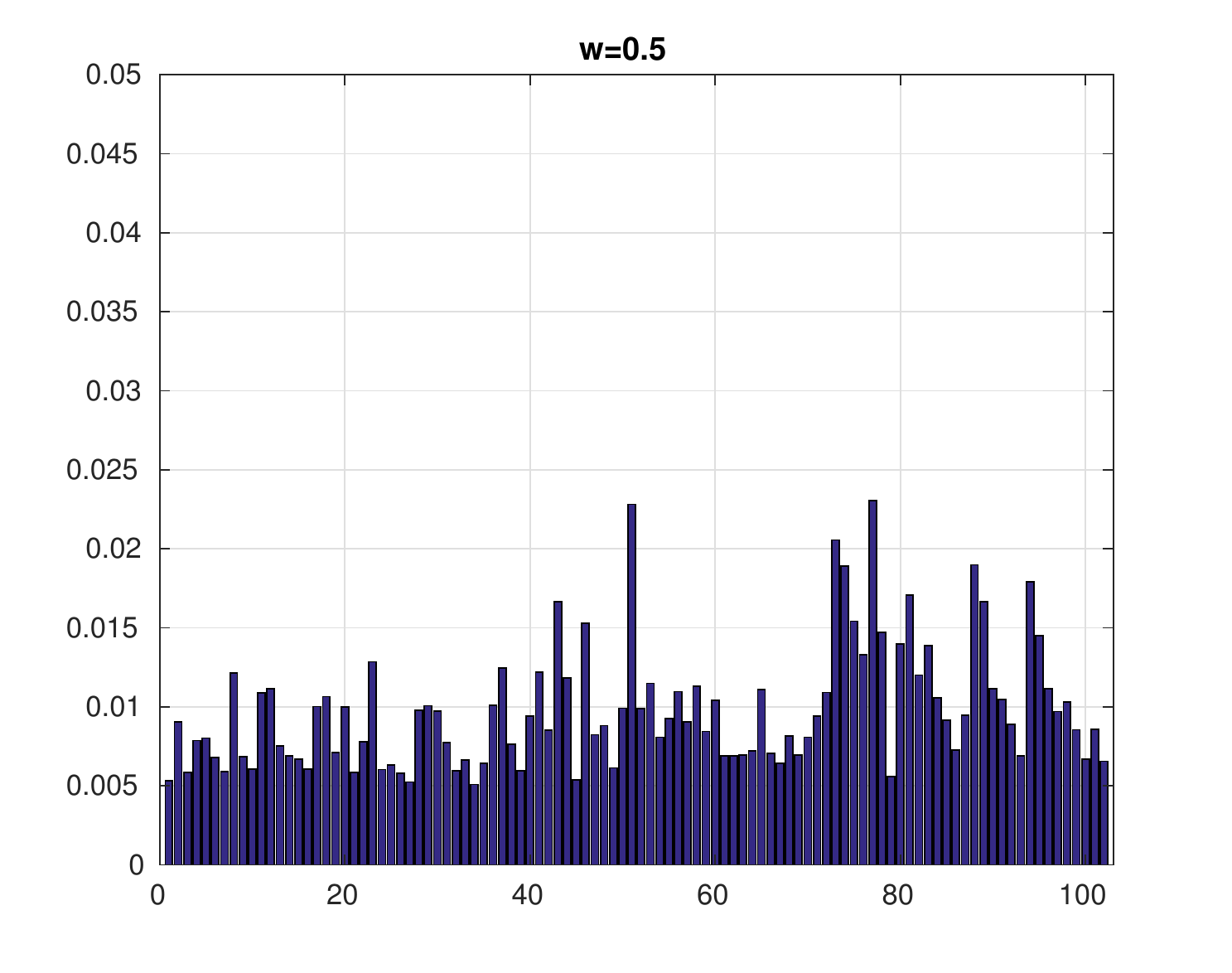}
}
\subfigure[w=0.7]{
\includegraphics[width=3.2cm, height=2.3cm]{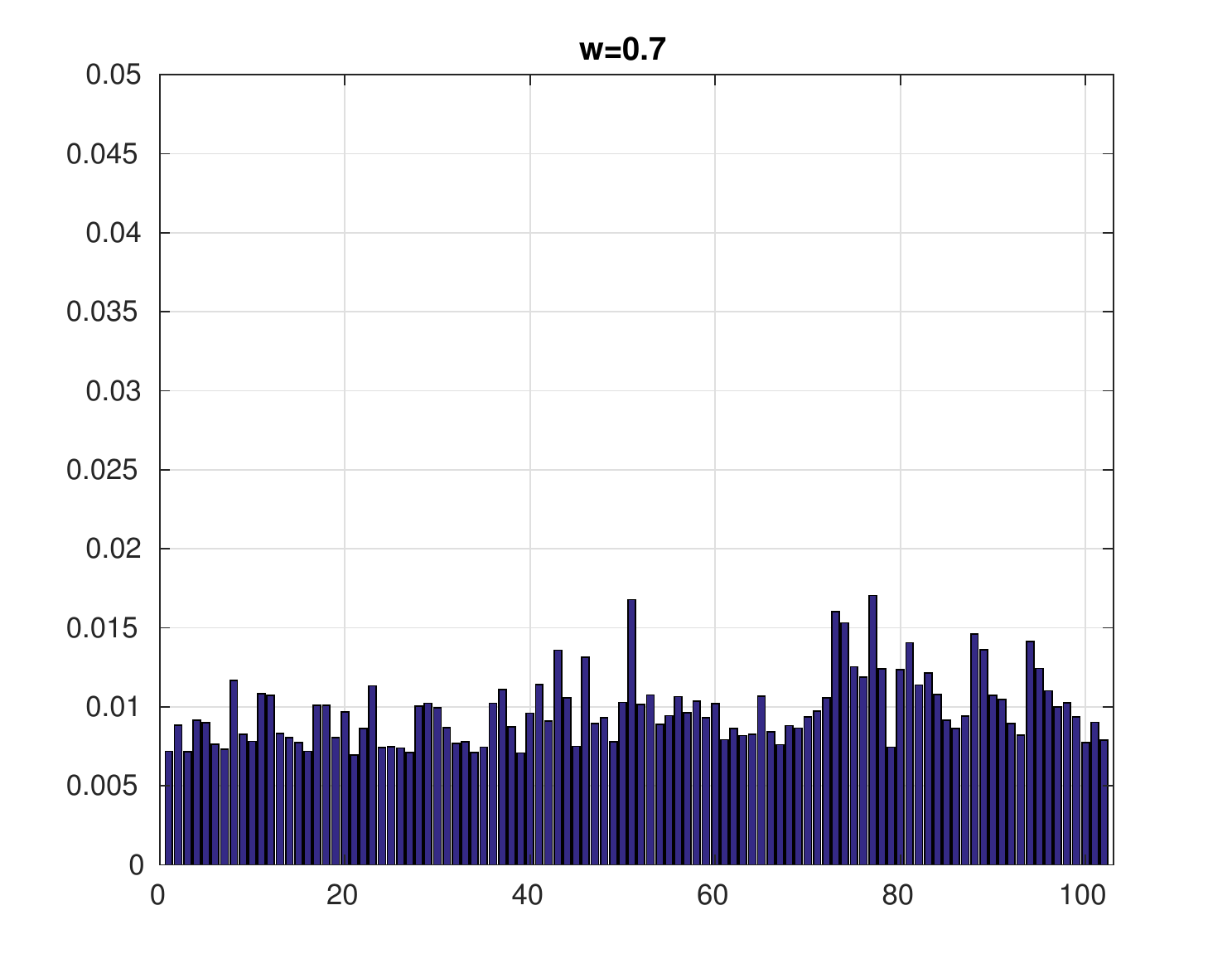}
}
\vspace{-10pt}

\subfigure[w=0.9]{
\includegraphics[width=3.2cm, height=2.3cm]{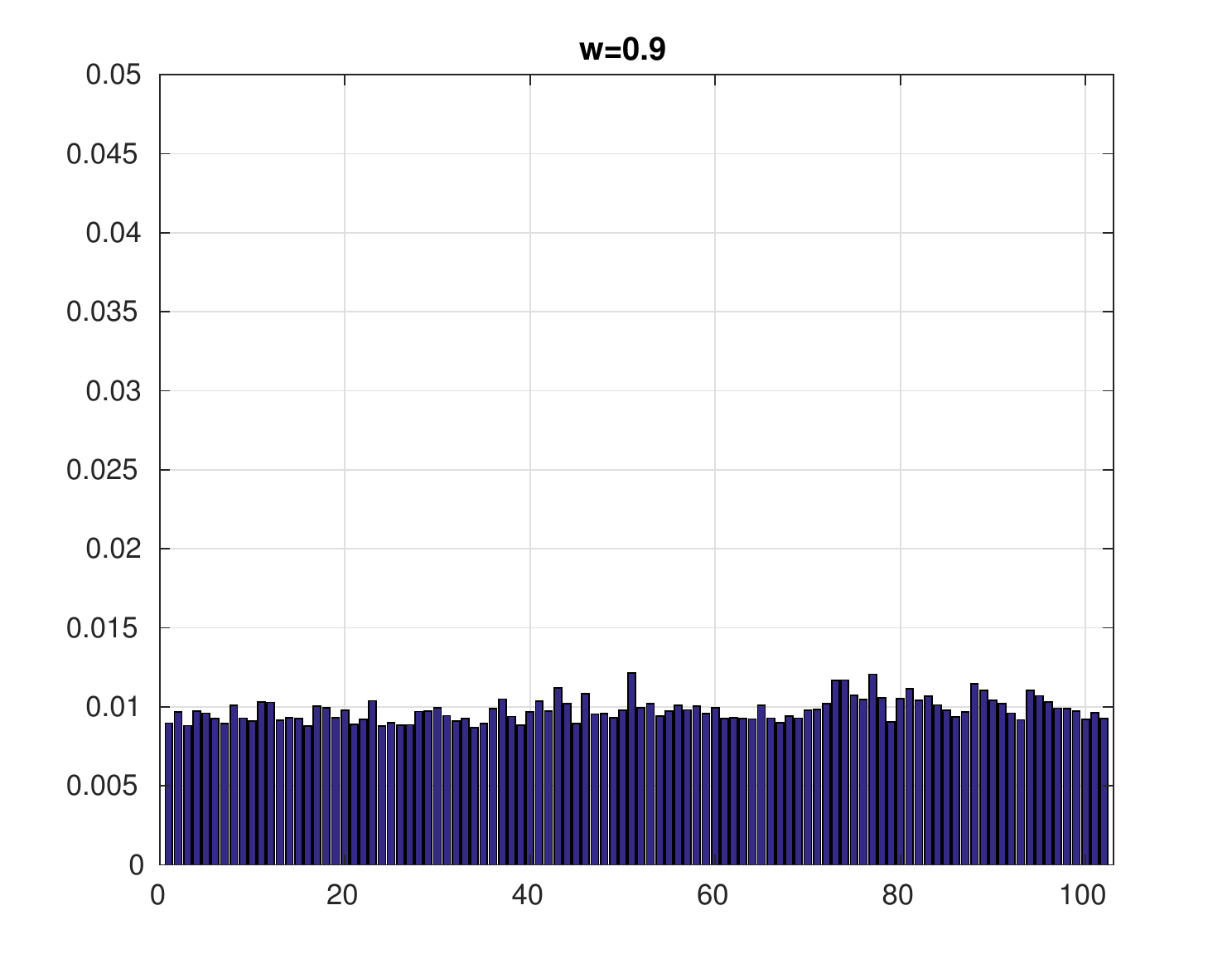}
}
\subfigure[w=1]{
\includegraphics[width=3.2cm, height=2.3cm]{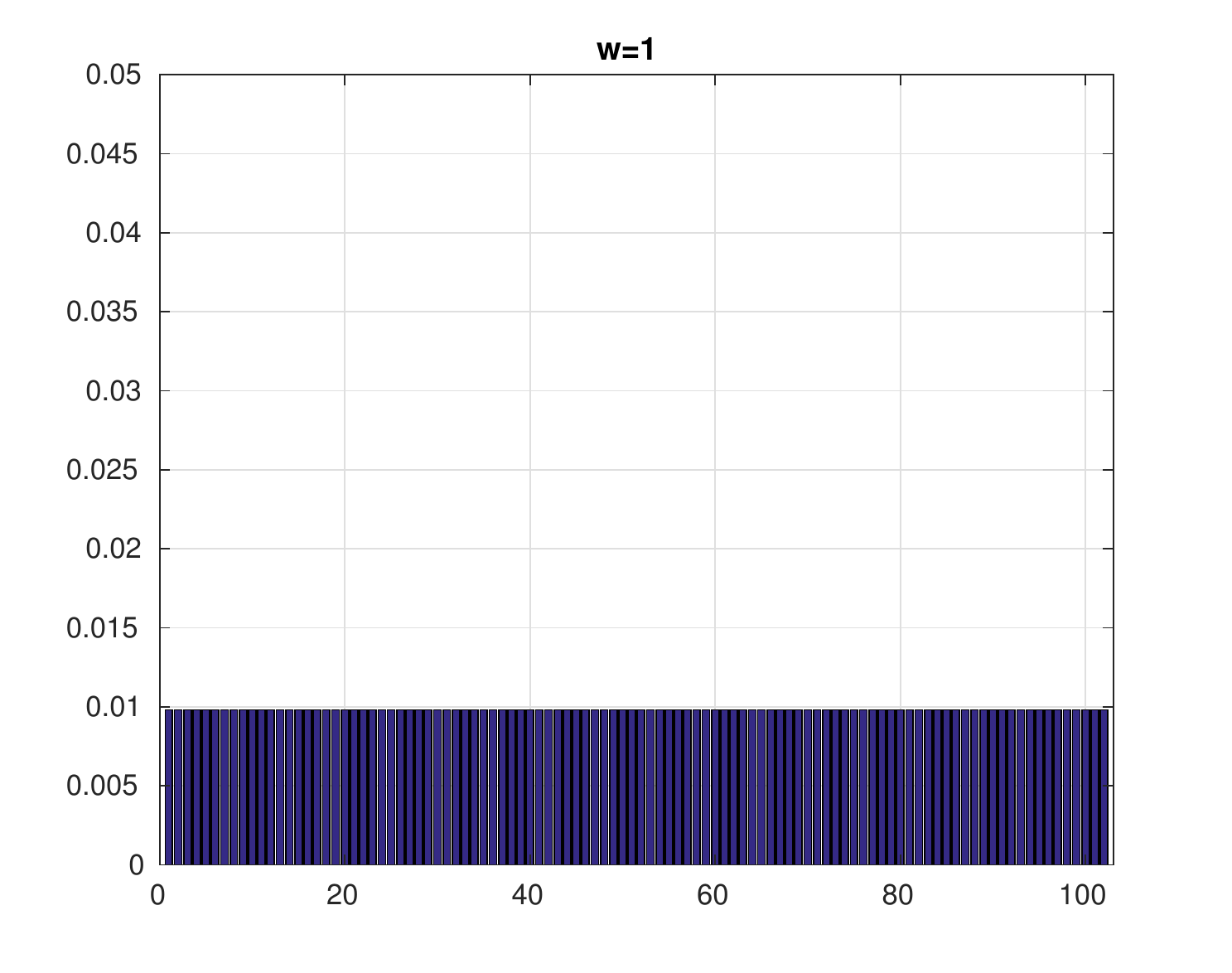}
}
\caption{The frequency of images in each class for Oxford Flower dataset, with $k=102$.
%$K=102$ is used here for comparison with stratification (a) shows the original training data. (b) shows the re-sampled data with $w=0$, which means the label information is not used, only the image off-the-shelf CNN feature is used. (c) -(g) shows the re-sampling results with different $w$ values and (h) shows the case with $w=1$, which means only the class label information is used. In this case, it is equivalent to stratified sampling.
}
\label{fig:sample_diff_weights}
\end{center}
\end{figure}

Figure \ref{fig:sample_diff_weights} shows the bar plots of the frequency of images in each class for Oxford Flower dataset using the number of classes as the mini-batch size. With this setting, we can see that when $w=1$, DM-SGD is equivalent to StS.

% {
% \bibliography{ref}
% \bibliographystyle{plain}
% }
% \end{document}
% \begin{align*}
% &\left|\begin{matrix}
% L_{ii} & L_{ij}\\
% L_{ji} & L_{jj}\\
% \end{matrix}\right| \\
% &= L_{ii} L_{jj} - L_{ij}L_{ji}
% \end{align*}

\end{document}